%% file: main.tex
\documentclass[letterpaper]{article} % DO NOT CHANGE THIS
\usepackage{aaai23}  % DO NOT CHANGE THIS
\usepackage{times}  % DO NOT CHANGE THIS
\usepackage{helvet}  % DO NOT CHANGE THIS
\usepackage{courier}  % DO NOT CHANGE THIS
\usepackage[hyphens]{url}  % DO NOT CHANGE THIS
\usepackage{graphicx} % DO NOT CHANGE THIS
\usepackage{natbib}  % DO NOT CHANGE THIS AND DO NOT ADD ANY OPTIONS TO IT
\usepackage{caption} % DO NOT CHANGE THIS AND DO NOT ADD ANY OPTIONS TO IT
\usepackage{algorithm}
\usepackage{algorithmicx,algpseudocode}

\usepackage{newfloat}
\usepackage{listings}
\DeclareCaptionStyle{ruled}{labelfont=normalfont,labelsep=colon,strut=off} % DO NOT CHANGE THIS
\floatstyle{ruled}
\newfloat{listing}{tb}{lst}{}
\floatname{listing}{Listing}
\title{Constrained Pure Exploration Multi-Armed Bandits with a Fixed Budget}
\author{
    Fathima Zarin Faizal, Jayakrishnan Nair
}
\affiliations{
    Indian Institute of Technology Bombay\\
    \{fathima,jayakrishnan.nair\}@ee.iitb.ac.in
}

\usepackage{bibentry}
\newcommand{\ra}{\rightarrow}

\newcommand{\prob}[1]{\mathbb{P}\left(#1\right)}

\newcommand{\R}{\mathbb{R}}

\newcommand{\ignore}[1]{}
\newcommand{\C}{\mathcal{C}}
\newcommand{\G}{\mathcal{G}}

\newboolean{showcomments}
\setboolean{showcomments}{true}
\newcommand{\JN}[1]{\ifthenelse{\boolean{showcomments}} {\textcolor{red}{(JN says: #1)}} {} }
\newcommand{\FZF}[1]{\ifthenelse{\boolean{showcomments}} {\textcolor{blue}{(Fathima says: #1)}} {} }

\usepackage{amsmath, mathtools,amssymb,amsthm,comment,subcaption,pgfplots}
\pgfplotsset{compat=newest}
\newtheorem{theorem}{Theorem}
\newtheorem{conjecture}{Conjecture}

\newtheorem{lemma}[theorem]{Lemma}
\newtheorem{definition}{Definition}[section]
\DeclareMathOperator*{\argmax}{arg\,max}
\DeclareMathOperator*{\argmin}{arg\,min}
\newcommand\numberthis{\addtocounter{equation}{1}\tag{\theequation}}
\newcommand{\algo}{\textsc{Constrained-SR}}
\usetikzlibrary{intersections,calc, patterns, shapes, arrows,positioning,quotes}
\allowdisplaybreaks[4]

\begin{document}

\maketitle

\begin{abstract}
  We consider a constrained, pure exploration, stochastic multi-armed
  bandit formulation under a fixed budget. Each arm is associated with
  an unknown, possibly multi-dimensional distribution and is described
  by multiple attributes that are a function of this distribution. The
  aim is to optimize a particular attribute subject to user-defined
  constraints on the other attributes. This framework models
  applications such as financial portfolio optimization, where it is
  natural to perform risk-constrained maximization of mean return. We
  assume that the attributes can be estimated using samples from the
  arms' distributions and that these estimators satisfy suitable
  concentration inequalities. We propose an algorithm called
  \textsc{Constrained-SR} based on the Successive Rejects framework,
  which recommends an optimal arm and flags the instance as being
  feasible or infeasible. A key feature of this algorithm is that it
  is designed on the basis of an information theoretic lower bound for
  two-armed instances. We characterize an instance-dependent upper
  bound on the probability of error under \textsc{Constrained-SR},
  that decays exponentially with respect to the budget. We further
  show that the associated decay rate is nearly optimal relative to an
  information theoretic lower bound in certain special cases. 
\end{abstract}

\input{sections/intro}
\input{sections/prob}
\input{sections/lb}

\input{sections/algo}

\input{sections/numerics}
\input{sections/conclusion}

%\section{Acknowledgments}

\bibliography{main}
%\newpage
\appendix
\input{sections/appendix_A}
\input{sections/appendix_C}
\input{sections/appendix_D}

\end{document}

%% file: sections/intro.tex
\section{Introduction}
\label{sec: intro}

The aim of the pure exploration, stochastic, multi-armed bandit (MAB)
problem is to identify, via exploration, the optimal arm among a given
basket of arms. Here, each arm is associated with an a priori unknown
probability distribution, and the optimal arm is classically defined
as one that optimizes a certain attribute associated with its
distribution (for example, the mean). However, in practical
applications, there is rarely just a single arm attribute that is of
interest. For example, in clinical trials, one might be interested in
not just the the efficacy of a treatment protocol, but also its cost
and the severity of its side effects. In portfolio optimization, one
is interested in not just the expected return of a candidate
portfolio, but also the associated variability/risk.

The classical approach in the MAB literature for handling multiple
constraints is to combine them into a single objective, often via a
linear combination \citep{vakili2016,kagrecha2019}. For example, in
portfolio optimization, the optimization of a linear combination of
expected return and its variance is often recommended
\citep{sani2012}. However, the main drawback of this approach is that
there is typically no sound way of determining the weights for this
linear combination. After all, can one equate the `value' of a unit
decrease in expected return of a portfolio to the `value' of a unit
decrease in the return variance in a scale-free manner?  Given that
the mean-variance landscape across the arms is a priori unknown, a
certain choice of arm objective might result in the `optimal' arm
having either an unacceptably low expected return, or an unacceptably
high variability.

An alternative approach for handling multiple arm attributes is to
pose the choice of optimal arm as a constrained optimization
problem. Specifically, the optimal arm is defined as the one that
optimizes a certain attribute, subject to constraints on other
attributes of interest. This avoids the `apples to oranges'
translation required in order to combine multiple attributes into a
single objective. Returning to our portfolio optimization example,
this approach would (potentially) define the optimal arm/portfolio as
the one that optimizes expected return subject to a prescribed risk
appetitite.

In this paper, we analyse such a constrained stochastic MAB
formulation, in the fixed budget pure exploration
setting. Specifically, each arm is associated with a (potentially
multi-dimensional) probability distribution. We consider two
attributes, both of which are functions of the arm distribution. The
optimal arm is then defined as one that minimizes one attribute
(henceforth referred to as the \emph{objective attribute}), subject to
a prescribed constraint on the other attribute (henceforth referred to
as the \emph{constraint attribute}).\footnote{We consider only a
single constraint attribute in this paper. The generalization to
multiple constraint attributes is straightforward, but cumbersome.}
Crucially, we make no limiting assumptions on the class of arm
distributions, or on the specific attributes considered. Instead, we
simply assume that the arm attributes can be estimated from samples
obtained from arm pulls, with reasonable concentration guarantees
(details in Section~\ref{sec: prob}).

While the unconstrained (single attribute) pure exploration MAB
formulation is well studied in the fixed budget setting, the
algorithms and lower bounds for this case do not generalize easily to
the constrained formulation described above. For example, the best
known algorithms for the unconstrained case divide the exploration
budget into phases, and eliminate/reject one or more arms at the end
of each phase (for example, the \emph{Successive Rejects} algorithm
by~\cite{audibert-bubeck}, and the \emph{Sequential Halving} algorithm
by \cite{Karnin2013}). The last surviving arm is then flagged as
optimal. While the exact `rejection schedule' differs across state of
the art algorithms, the decision on which arm(s) to reject is itself
straightforward, given a single, scalar arm attribute.

However, in the presence of multiple constraints, the decision on
which arm(s) to reject is non-trivial, given that estimates of both
attributes of each surviving arm must be taken into consideration. A
naive strategy is to focus on first rejecting the arms that appear
`infeasible' (i.e., arms whose constraint attribute estimates violate
the prescribed threshold), and then reject those arms that appear
`feasible but suboptimal.' However, this strategy can be far from
optimal (see Section~\ref{sec: numerics}). Instead, the approach we
propose exploits an information theoretic lower bound for two-armed
instances. Specifically, this lower bound motivates the definition of
certain suboptimality gaps between pairs of arms. We then reject arms
sequentially based on empirical estimates of these pairwise gaps,
along with a specific tie-breaking rule. This novel approach, which we
formalize as the \algo\ algorithm, is the main contribution of this
paper.

This paper is organized as follows. After a brief survey of the
related literature, we formally describe the constrained pure
exploration MAB formulation in Section~\ref{sec: prob}. In
Section~\ref{sec: lb}, we derive an information theoretic lower bound
for two-armed instances, which leads to a conjecture on the lower
bound for the general $K$-armed case. The \algo\ algorithm is
described and analyzed in Section~\ref{sec: algo}. Finally, we provide
a numerical case study in Section~\ref{sec: numerics}, and conclude in
Section~\ref{sec:conclusion}.

Throughout the paper, references to the appendix (mainly for proofs of
certain technical results) point to the appendix in the supplementary
materials document.

\noindent {\bf Related literature:} There is a substantial literature
on the multi-armed bandit problem. We refer the reader to excellent
textbook treatments \cite{bubeck2012,lattimore} for an overview. In
this review, we restrict attention to (the few) papers that consider
MABs with multiple attributes.

\cite{drugan2013,yahyaa2015} consider the \emph{Pareto frontier} in
the attribute space; the goal in these papers is to play all
Pareto-optimal arms equally often. Another useful notion is
\emph{lexicographic optimality}, where the attributes are `ranked'
with `less important' attributes used to break ties in values of `more
important' attributes (see
\cite{ehrgott2005}). \cite{tekin2018,tekin2019} apply the notion of
lexicographic optimality to contextual MABs.

The paper closest to the present paper
is~\cite{kagrecha2020constrained}, which analyses a similar
constrained MAB formulation, but in the \emph{regret minimization}
setting. This paper proposes a UCB-style algorithm for this problem,
and establishes information theoretic lower bounds. The follow-up
paper~\cite{chang2020risk} proposes a Thompson Sampling based
variant. Special cases of the constrained MAB problem (with a risk
constraint) are considered in the pure exploration \emph{fixed
confidence} setting
in~\cite{david2018,hou2022almost}.~\cite{chang2020} considers an
average cost constraint (each arm has a cost distribution that is
independent of its reward distribution), pursuing the weaker goal of
\emph{asymptotic optimality}. A linear bandit setting is considered
in~\cite{pacchiano2021} under the assumption that there is at least
one arm which satisfies the constraints. Finally,~\cite{amani2019,
  moradipari2019} consider the problem of maximizing the reward
subject to satisfying a linear `safety' constraint with high
probability. None of the above mentioned papers considers the
\emph{fixed budget} pure exploration setting considered
here. Additionally, all the papers above (with the exception of
\cite{kagrecha2020constrained}) implicitly assume that the instance is
feasible; the present paper explicitly addresses the practically
relevant possibility that the learning agent may encounter an instance
where no arm meets the prescribed constraint(s).

%Interestingly, in contrast with the fixed budget pure exploration
%setting considered here, algorithms for regret minimization in
%unconstrained (single attribute) MABs generalize easily to the
%constrained formulation.

%% file: sections/prob.tex
\section{Problem formulation}
\label{sec: prob}

In this section, we describe the formulation of the constrained
stochastic MAB problem studied here. We consider the fixed budget,
pure exploration framework; the MAB instance is parameterized by a
budget of $T$ rounds (a.k.a., arm pulls) and $K$ arms labelled $1,
\ldots, K$, each of which is associated with an a priori unknown
probability distribution. We consider a \emph{constrained} setting,
wherein the optimal arm is defined to be the one that optimizes a
certain attribute, subject to a constraint on another attribute. In a
nutshell, the goal of the learner (a.k.a., algorithm) is to identify
the optimal arm in the instance, and also to flag the instance as
being feasible or infeasible (i.e., indicating whether any or none of
the arms meets the constraint, respectively), using the budget of $T$
arm pulls for exploration.
%the constrained stochastic MAB problem is, under the fixed budget
%setting that is considered here, to either identify the arm that
%optimizes the objective subject to a constraint on another attribute
%or to declare the instance as being unsolvable, given a fixed number
%of rounds.
The rest of this section is devoted to formalizing this problem.

Each arm $i$ is associated with a possibly multi-dimensional
distribution $\nu(i)$. These distributions are unknown to the
learner. Let $\C$ denote the space of arm distributions, i.e., $\nu(i)
\in \C$ for all~$i.$ We define the objective and constraint attributes
$\mu_1$ and $\mu_2,$ respectively, to be functions from $\C$ to $\R$.
We henceforth refer to $\mu_j(i) = \mu_j(\nu(i))$ as the value of
attribute $j$ ($j \in \{1,2\}$) associated with arm~$i,$ with $\mu(i)$
denoting the vector $(\mu_1(i),\mu_2(i)).$ The user specifies a
threshold $\tau \in \R$, which defines an upper bound for the
attribute $\mu_2.$ An \textit{instance} of this constrained MAB
problem is specified by $(\nu, \tau)$ where $\nu=(\nu(1), \ldots,
\nu(K))$. The arms for which the constraint is satisfied, i.e.,
$\mu_2(i) \leq \tau$, are called \textit{feasible arms}; and the set of
feasible arms is denoted by $\mathcal{K}(\nu)$. The instance
$(\nu,\tau)$ is said to be feasible if $\mathcal{K}(\nu) \neq
\emptyset$, and infeasible if $\mathcal{K}(\nu) = \emptyset$.

Consider a feasible instance. We define an arm to be \textit{optimal}
if it has the least value of $\mu_1(\cdot),$ subject to the constraint
$\mu_2(\cdot) \leq \tau$. For simplicity of exposition, we assume that
there is a unique optimal arm.\footnote{As is well understood in the
pure exploration, fixed budget setting, it is straightforward to
handle the generalization where there are multiple optimal arms.} We
formally denote the optimal arm as $$J([K]) = \argmin_{i \in
  \mathcal{K}(\nu)} \mu_1(i).$$ Here, $[K] = \{1,\ldots,K\}$.  Without
loss of generality, we assume $J([K]) = 1,$ i.e., arm~1 is optimal.
An arm~$i$ is said to be \textit{suboptimal} if $\mu_1(i)>\mu_1(1)$
(irrespective of whether it is feasible or not). Further, an arm~$i$
is said to be a \textit{deceiver} if $\mu_1(i) \leq \mu_1(1),$ but
$\mu_2(i) > \tau.$ The different types of arms in a feasible instance
are illustrated in Figure~\ref{fig: feasible}.

Next, consider an infeasible instance. In this case, an optimal arm is
defined as the one with the smallest value of $\mu_2(\cdot),$ i.e., the
one that is `least infeasible.' As before, we assume for simplicity
that there is a unique optimal arm, denoted by $$J([K]) = \argmin_{i
  \in [K]} \mu_2(i).$$ An infeasible instance is illustrated in
Figure~\ref{fig: infeasible}. 

In each round $t \in [T]$, the learner chooses an arm from the set of
arms $[K]$ and observes a sample drawn from the corresponding
distribution (independent of past actions and observations). At the
end of $T$ rounds, the learner outputs a tuple
$(\hat{J}([K]),\hat{F}([K]))$, where $\hat{J}([K]) \in [K]$ and
$\hat{F}([K])$ is either \texttt{True} or \texttt{False}. The output
$\hat{J}([K])$ is the learner's recommendation for the optimal
arm. The output $\hat{F}([K])$ is a Boolean flag that indicates
whether the learner deems the instance as being feasible (in which
case $\hat{F}([K])=$\texttt{True}) or infeasible (in which case
$\hat{F}([K])=$\texttt{False}). Let us denote by $F([K])$ the correct
value of the feasibility flag for the instance, i.e.,
$F([K])=\texttt{True}$ if the instance is feasible and
$F([K])=\texttt{False}$ otherwise. The algorithm is evaluated based on
its probability of error $e_T$, which is,
\begin{align*}
    e_T=\mathbb{P} \left (  \left \{ J([K]) \neq  \hat{J}([K]) \right \} \cup \left \{ (\hat{F}([K]) \neq F([K]) \right \} \right ).
\end{align*}
For notational simplicity, we have suppressed the dependence of $e_T$
on the algorithm and the instance. The goal is to design algorithms
that minimize the probability of error.
%As explained in \citet{audibert-bubeck}, when the rewards are bounded

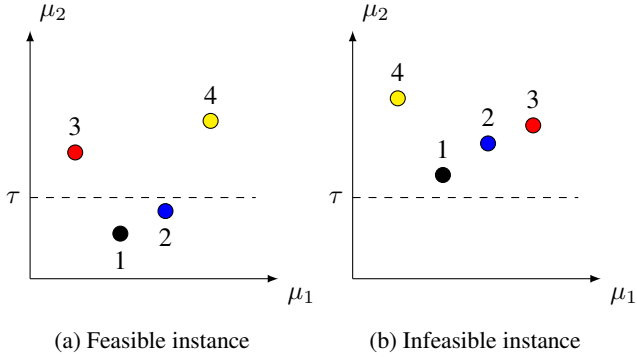
\begin{figure}
     \centering
     \begin{subfigure}{0.49\linewidth}
         \centering
  \begin{tikzpicture}[scale=0.6,    
  dot/.style = {circle, draw, fill=#1, inner sep=2pt}
  ]
    % axes
    \draw[draw,latex-] (0,5) +(0,0.5cm) node[above right] {$\mu_2$} -- (0,0);
    \draw[draw,-latex] (0,0) -- (5,0) -- +(0.5cm,0) node[below right] {$\mu_1$};
    
    % tau line
    \draw[dashed] (0,1.8) node[left] {$\tau$} -- (5,1.8);
    
    % points 
    \node[dot=black,  label=below:{1}]  at (2, 1) {};
    \node[dot=blue,   label=below:{2}]  at (3, 1.5) {};
    \node[dot=red,    label=3]  at (1, 2.8) {};
    \node[dot=yellow, label=4]  at (4, 3.5) {};
  \end{tikzpicture}
         \caption{Feasible instance}
         \label{fig: feasible}
     \end{subfigure}
     \hfill
     \begin{subfigure}{0.49\linewidth}
         \centering
  \begin{tikzpicture}[scale=0.6,    
  dot/.style = {circle, draw, fill=#1, inner sep=2pt}
  ]
    % axes
    \draw[draw,latex-] (0,5) +(0,0.5cm) node[above right] {$\mu_2$} -- (0,0);
    \draw[draw,-latex] (0,0) -- (5,0) -- +(0.5cm,0) node[below right] {$\mu_1$};
    
    % tau line
    \draw[dashed] (0,1.8) node[left] {$\tau$} -- (5,1.8);
    
    % points 
    \node[dot=black,  label=1]  at (2, 2.3) {};
    \node[dot=blue,   label=2]  at (3, 3) {};
    \node[dot=red,    label=3]  at (4, 3.4) {};
    \node[dot=yellow, label=4]  at (1, 4) {};
  \end{tikzpicture}

         \caption{Infeasible instance }
         \label{fig: infeasible}
     \end{subfigure}
        \caption{Panel~(a) shows a feasible instance. Arm 1 is
          optimal, arm 2 is feasible suboptimal, arm 3 is a deceiver,
          and arm 4 is infeasible suboptimal. Panel~(b) shows an
          infeasible instance (arm 1 is optimal).}
        \label{fig:three graphs}
\end{figure}

Finally, as stochastic MAB algorithms require estimators for the
attributes $\mu_1$ and $\mu_2$, which are functions of the data
samples of each arm, we assume the following concentration properties
for these estimators. Specifically, we assume that for $i \in \{1,2\}$
and distribution $G \in \mathcal{C}$, there exists an estimator
$\hat{\mu}_{i,n}(G)$ of $\mu_i(G)$ which uses $n$ i.i.d. samples
from~$G$, satisfying the following concentration inequality: There
exists $a_i>0$ such that for all $\Delta>0$,
\begin{align}
  \mathbb{P}\left(\left|\hat{\mu}_{i, n}(G)-\mu_{i}(G)\right| \geq \Delta\right) \leq 2 \exp \left(-a_{i} n \Delta^{2}\right).
  \label{eqn: conc_inequality}
\end{align}
Such concentration inequalities are commonly used for analyzing MAB
algorithms.\footnote{The standard practice when dealing with classical
  (unconstrained) MAB problems is to specify both the set~$C$ of arm
  distributions (for example, as the set of 1-subGaussians) and the
  attribute being optimized~$\mu_1$ (for example, the mean of the arm
  distribution). These choices then imply natural
  estimators~$\hat{\mu}_1$ and their corresponding concentration
  properties. In this work, to avoid working with a specific
  distribution class and a specific set of arm attributes, and to
  emphasize the generality of the proposed approach, we simply assume
  that attribute estimates satisfy concentration inequalities of the
  form~\eqref{eqn: conc_inequality}. Moreover, this particular form
  for the concentration inequality is assumed only for ease of
  exposition; changes to this form (as might be needed, for example,
  if the arm distributions are sub-exponential or heavy-tailed) lead
  to minor modifications to our algorithms and bounds.} For instance,
if the attributes can be expressed as expectations of sub-Gaussian or
bounded random variables (which are themselves functions of the arm
samples), concentration inequalities of the form \eqref{eqn:
  conc_inequality} would hold for the empirical average using the
Cram\'er-Chernoff bound or the Hoeffding inequality respectively
(refer Chapter~5 of \citet{lattimore}). Several risk measures like
Value-at-Risk (VaR) and Conditional Value-at-Risk (CVaR) also admit
estimators with concentration properties of the form~\eqref{eqn:
  conc_inequality}; see \cite{wang2010,cassel,kolla2019,bhat2019}.
  
Finally, we define the notion of \emph{consistency} of an
algorithm. An algorithm is said to be consistent over~$(\C,\tau)$ if,
for all instances of the form~$(\nu,\tau),$ where $\nu \in \C^K,$
$\lim_{T \ra \infty}e_T = 0.$

\begin{comment}
For a deceiver arm $i$, we define the gap
$\Delta_i=\mu_2(i)-\tau$. For a suboptimal feasible arm $i$, we define
the gap $\Delta_i=\mu_1(i)-\mu_1(1)$, which aligns with the
suboptimality gap definition from the one-dimensional mean
minimization case. For a suboptimal infeasible arm $i$, we define the
gap $\Delta_i=\max \{\mu_1(i)-\mu_1(1),\mu_2(i)-\tau\}$. The gap for
any arm $i$ is defined as $\mu_2(i)-\underset{k \in [K]}{\min}
\mu_2(k)$.

For $j\in\{1,2\}$, $1 \leq i \leq K$ and $t \in [T]$, let $T_{i}(t)$
denote the number of times arm $i$ was pulled till round $t$ and let
$\hat{\mu}_{j,s}(i)=\frac{1}{s}\sum_{t=1}^{s}X_{i, t}^{j}$ be the
empirical mean of the $j^{\textrm{th}}$ dimension of arm $i$ after $s$
pulls.
\end{comment}

%% file: sections/lb.tex
\section{Lower bound}
\label{sec: lb}

In this section, we provide an information theoretic lower bound on
the probability of error under any algorithm, for a class of two-armed
Gaussian bandit instances. We then extrapolate this bound to
conjecture a lower bound for the general $K$-armed case. Crucially,
the lower bound for the two-armed case forms the basis for our
algorithm design (see Section~\ref{sec: algo}).

First, we define some sub-optimality gaps that will be used to state
our lower bounds, and also later when we discuss algorithms. Given two
arms $i$ and $j,$ we say $i \succ j$ if~$i$ is an optimal arm in a
two-armed instance consisting only of arms~$i$ and~$j.$ For $i \succ
j,$ define $\delta(i,j) = \delta(\mu(i),\mu(j))$ as
follows.\footnote{We abuse notation and use~$\delta(\mu(i),\mu(j))$ in
place of $\delta(i,j)$ when we need to emphasize the dependence
of~$\delta$ on the attribute values of arms~$i$ and~$j.$} $\delta(i,j) :=$
\begin{equation*}
  \displaystyle
   \left\{
  \begin{array}{ll}
    \sqrt{a_1} \left ( \mu_{1}(j)- \mu_{1}(i) \right ), & \text{ if } i, j \in \mathcal{K}(\nu),\\ 
    \sqrt{a_2} \left (\mu_2(j)-\tau \right), & \text{ if } i \in \mathcal{K}(\nu) \text{ and} \\ &  \quad j \notin \mathcal{K}(\nu) \text{ is a deceiver,} \\
    \max \bigl\{ \sqrt{a_2} \left ((\mu_2(j) - \tau \right ), & \text{ if } i \in \mathcal{K}(\nu) \text{ and}  \\
    \quad \sqrt{a_1} \left  (\mu_1(j)-\mu_1(i) \right ) \bigr\},
    & \quad j \notin \mathcal{K}(\nu) \text{ is suboptimal,} \\
    \sqrt{a_2} (\mu_2(j)-\mu_2(i)), & \text{ if } i,j \notin \mathcal{K}(\nu)
  \end{array}
  \right..
  \label{eq:delta_def}
\end{equation*}

Next, for $i \succ j,$ define $\Delta(i,j) = \Delta(\mu(i),\mu(j))$ as
follows.
\begin{equation*}
  \displaystyle
  \Delta(i,j) :=
    \min  \{ \sqrt{a_2}\ |\tau - \mu_2(i)|,  \delta(i,j)  \}.
\end{equation*}

\ignore{
  We define the following gap function $\Delta(\cdot, \cdot)$
  between any two arms that are a measure of how hard it is for any
  algorithm to give the correct output on that two-armed instance.

\begin{definition}
$\Delta(\cdot, \cdot)$ and $\delta(\cdot,\cdot)$ are functions from $[K] \times [K]$ to $\mathbb{R}$ defined for a given instance. The first argument is always $J([K])$ and the gaps are defined with respect to this arm. Let $j \in [K]$ be any arm.\\ Let the instance be feasible. If $j$ is a feasible suboptimal arm, we define
 \begin{align*}
 \delta(J([K]),j)=\sqrt{a_1} \left ( \mu_{1}(j)- \mu_{1}(J([K])) \right ).
 \end{align*}
If $j$ is a deceiver arm, we define
 \begin{align*}
 \delta(J([K]),j)=\sqrt{a_2} \left (\mu_2(j)-\tau \right ).
 \end{align*}
If $j$ is an infeasible suboptimal arm, we define
 \begin{align*}
 \delta(J([K]),j)=\max \{ \sqrt{a_2} \left ((\mu_2(j) - \tau \right ) , \sqrt{a_1} \left  (\mu_1(j)-\mu_1(J([K])) \right ) \}.
 \end{align*}
Let the instance be infeasible. Then, we define
\begin{align*}
 \delta(J([K]),j)=\sqrt{a_2} \left ( \mu_{2}(j)- \mu_{2}(J([K])) \right ).
\end{align*}
Let $j$ be any arm. $\Delta(\cdot, \cdot)$ is defined as:
\begin{align*}
\Delta(J([K]),j) &=\min \left \{ \sqrt{a_1} \left ( \tau - \mu_2(J([K])) \right ), \delta(J([K]),j) \right \},
\shortintertext{if the instance is feasible, else:}
\Delta(J([K]),j) &=\delta(J([K]),j).
\end{align*}
\end{definition}
} %end ignore

As we will see, the smaller the value of $\Delta(i,j),$ the harder it
is for a learner to identify the optimal arm~$i$ in a two-armed
instance consisting of arms~$i$ and~$j.$ Thus, one may
interpret~$\Delta(i,j)$ as the `suboptimality gap' of arm~$j$
(relative to arm~$i$); note that this gap depends on the values of the
objective attributes, the constraint attributes, the threshold~$\tau,$
and also the concentration parameters~$a_1,a_2.$ For example, if $i$
is feasible and $j$ is feasible and suboptimal, $\Delta(i,j) =
\min\{\sqrt{a_2}\left (\tau - \mu_2(i) \right), \sqrt{a_1} \left (
\mu_{1}(j)- \mu_{1}(i) \right )\}.$ Thus, the closer~$i$ is to the
constraint boundary, and the smaller the gap between~$i$ and~$j$ in
the objective attribute, the harder it is to identify~$i$ as the
optimal arm in this pair. The gaps for the other cases can be
interpreted in a similar manner.

\ignore{For instance, for a feasible suboptimal arm $j$ in a feasible
  instance, a larger gap $\delta(\cdot, \cdot)$ with respect to the
  optimal arm would imply that their objective attributes are
  ``farther away", thus making it easier for any algorithm to
  differentiate between the optimal arm and arm $j$. A larger gap
  $\Delta(\cdot, \cdot)$ with respect to the optimal arm would imply
  that the constraint attribute of the optimal arm lies sufficiently
  ``far away" from $\tau$, thus reducing the likelihood of a good
  learner concluding that $J([K])$ is an infeasible arm.} 

We are now ready to state our information theoretic lower
bound. Consider the class of arm distributions $\mathcal{D}$, which
consists of 2-dimensional Gaussian distributions with covariance
matrix $\Sigma =
\textrm{diag}\left(\frac{1}{2a_1},\frac{1}{2a_2}\right).$ Attribute
$\mu_1$ is the mean of the first dimension, while attribute $\mu_2$ is
the mean of the second dimension. Note that the empirical average
estimator satisfies~\eqref{eqn: conc_inequality} for both attributes.

%As defined in \citet{kaufmann16}, given an environment $\mathcal{E}$,
%a family of algorithms $A$ is called \emph{consistent} if, for every
%choice of $\nu \in \mathcal{E}$, $e_T(\nu)$ tends to zero as $T$ tends
%to infinity.

%The following theorem summarizes the lower bound on the probability
%of error for any instance $\nu \in \mathcal{D}$. The proof of this
%theorem closely follows that of Theorem 12 in \cite{kaufmann16} and
%can be found in <\textsc{insert ref to section}>.

\begin{theorem}
  \label{thm: 2 arms lb}
Let $\nu$ be a two-armed bandit instance where $\nu(i) \in
\mathcal{D}$ for $i \in \{1,2\},$ with attribute~$\mu_1$ being the
mean of the first dimension of the arm distribution, and
attribute~$\mu_2$ being the mean of the second dimension of the arm
distribution. Under any consistent algorithm,
\begin{align*}
\limsup_{T \rightarrow \infty}-\frac{1}{T} \log e_{T}(\nu) \leq (\Delta(1,2))^2,
\end{align*}
where arm~1 is taken to be the optimal arm (without loss of generality).
\end{theorem}

The proof of Theorem~\ref{thm: 2 arms lb} can be found in
Appendix~\ref{app: 2_arm_lb_proof}. Note that Theorem~\ref{thm: 2 arms
  lb} provides an upper bound on the (asymptotic) exponential rate of
decay of the probability of error as $T \ra \infty.$ Specifically, the
decay rate can be at most $\Delta^2(1,2).$ This formalizes the
interpretation of $\Delta(1,2)$ as a `suboptimality gap' between arm~2
and arm~1. It is instructive to see which aspects of the arm
attributes influence this suboptimality gap. For example, if both
arms~1 and~2 are feasible, then $\Delta(1,2)$ depends on the
\emph{optimality gap} (i.e., $\mu_1(2) - \mu_1(1)$) and the
\emph{feasiblity gap} of arm 1 (i.e., $\tau - \mu_2(1)$) but not on
the \emph{feasibility gap} of arm~2 (i.e., $\tau - \mu_2(2)$). On the
other hand, if arm~1 is feasible and arm~2 is a deceiver, then
$\Delta(1,2)$ depends on the \emph{feasibility gap} of arm~1 (i.e.,
$\tau - \mu_2(1)$) and the \emph{infeasibility gap} of arm~2 (i.e.,
$\tau - \mu_2(2)$), but not on the gap between the objective
attributes. In Section~\ref{sec: algo}, we design an algorithm that
eliminates arms from consideration sequentially based on estimates of
these (pairwise) suboptimality gaps.
 
\ignore{
Thus, the rate of decay of the probability of error depends on the gap
between the optimal and the non-optimal arm. Note that if $\tau$ tends
to infinity in \eqref{thm: 2 arms lb}, the lower bound obtained is the
same as that in Theorem 12 of \citet{kaufmann16}. For instance,
consider a feasible instance where the non-optimal arm is feasible
suboptimal. Theorem \ref{thm: 2 arms lb} says that the probability of
error for any consistent algorithm on this instance does not depend
upon the constraint attribute of $J([2])^{\mathsf{c}}$. Similarly, if
$J([2])^{\mathsf{c}}$ is a deceiver arm, then the probability of error
for any consistent algorithm on this instance does not depend upon the
objective dimension of $J([2])^{\mathsf{c}}$.
}

Based on Theorem \ref{thm: 2 arms lb}, and results
from~\cite{audibert-bubeck} on the classical (unconstrained) MAB
problem, we conjecture the following extension of Theorem~\ref{thm: 2
  arms lb} to the case of $K$ arms as follows. Taking arm~1 to be the
optimal arm without loss of generality, define $H_1:=\sum_{i=2}^{K}
\frac{1}{\Delta^2(1,i)}$.
\begin{conjecture}
  \label{thm: K arms lb}
  Let $\nu$ be a $K$-armed bandit instance where $\nu(i) \in
  \mathcal{D}, i \in [K],$ with attribute $\mu_1$ being the mean of
  the first dimension of the arm distribution, and attribute $\mu_2$
  being the mean of the second dimension of the arm
  distribution. Under any consistent algorithm,
  \begin{align}
    \limsup_{T \rightarrow \infty}-\frac{1}{T} \log e_{T}(\nu) \leq
    \frac{d}{H_1}, 
  \end{align}
  where~$d$ is a universal positive constant.
\end{conjecture}
%It is possible to prove this lower bound for special cases such as
%when all arms are feasible or when all arms are infeasible (see
%Appendix~\ref{app: K arms lb}).
%
The main challenge in proving this conjecture for a general $K$-armed
instance is that exising lower bound approaches for the unconstrained
setting \citep{audibert-bubeck,kaufmann16,carpentier2016tight} do not
generalize to the constrained setting.
%The primary challenge in proving this conjecture for more
%general classes of instances lies in the lack of asymptotic lower
%bounds for general classes of $K$-armed instances even in the
%classical (unconstrained) MAB fixed budget problem (the closest being
%Theorem 16 of \cite{kaufmann16}).
As per Conjecture~\ref{thm: K arms lb}, $H_1$ can be interpreted as a
measure of the hardness of the instance under consideration. Indeed,
this definition of~$H_1$ agrees with the hardness measure that appears
in lower bounds for the classical (unconstrained) MAB problem (also
denoted~$H_1;$ see \cite{audibert-bubeck,kaufmann16}) when~$\tau \ra
\infty.$

%The rate of decay of the probability of error of any consistent
%algorithm is inversely proportional to $H$. Thus, $H$ is a measure of
%how hard it is for a consistent algorithm to recommend a wrong
%output.

%% file: sections/algo.tex
\section{The \textsc{Constrained-SR} algorithm}
\label{sec: algo}

In this section, we propose the \algo\ algorithm for the constrained
MAB problem posed in Section~\ref{sec: prob}, and provide a
performance guarantee via an upper bound on the probability of error
under this algorithm. This upper bound compares favourably with the
information theoretic lower bound conjectured in Section~\ref{sec: lb}
(see Conjecture~\ref{thm: K arms lb}), suggesting that the
\algo\ algorithm is nearly optimal. Importantly, the design of the
\algo\ algorithm is motivated by our information theoretic lower bound
for the two-armed case (see Theorem~\ref{thm: 2 arms lb});
\algo\ rejects arms sequentially based on \emph{estimates} of the same
pairwise suboptimality gaps that appear in the lower bound.

\noindent {\bf Algorithm description:} The \algo\ algorithm is based on
the well-known Successive Rejects (SR) framework proposed by
\citet{audibert-bubeck}. Informally, SR runs over $K-1$ phases; at the
end of each phase, one arm (the one that looks empirically `worst') is
rejected from consideration. Specifically, SR defines positive
integers $n_1,n_2,\ldots,n_{K-1},$ such that $n_1 < n_2 < \cdots <
n_{K-1}$ and $n_1 + n_2 + \cdots + n_{K-2} + 2 n_{K-1} \leq T$ (see
Algorithm~\ref{algo: pairwise} for the details). In phase~$k,$ each of
the surviving $K-k+1$ arms is pulled $n_k - n_{k-1}$ times. (This
means that by the end of phase~$k,$ each surviving arm has been
pulled~$n_k$ times.) The sole arm that survives at the end of
phase~$K-1$ is declared to be optimal, and the instance is flagged as
feasible (respectively, infeasible) if this surviving arm `appears'
feasible (respectively, infeasible).

\algo\  (formal description as Algorithm~\ref{algo: pairwise}) differs
from SR in the criterion used to reject an arm at the end of each
phase. Note that the classical SR algorithm is designed for a single
attribute; this makes the choice of the empirically `worst' arm
obvious. In contrast, the elimination criterion for our constrained
MAB problem should depend on estimates of \emph{both} attributes for
each surviving arm. The \algo\  algorithm does this as follows: Let
$\hat{J}(A_k)$ denote the arm that `appears' optimal at the end of
phase~$k,$ where $A_k$ denotes the set of surviving arms at the
beginning of phase~$k.$ Formally, letting~$\hat{\mu}^k_j(i)$ denote
the estimate of attribute~$j$ for arm~$i$ at the end of phase~$k,$
\begin{equation}
  \label{eq:emp_opt_arm}
  \displaystyle
  \hat{J}(A_k) =
  \left\{
  \begin{array}{ll}
    \displaystyle
    \argmin_{i \in A_k \colon \hat{\mu}^k_2(i) \leq \tau} \hat{\mu}^k_1(i), & 
    \{i \in A_k \colon \hat{\mu}^k_2(i) \leq \tau\} \neq \emptyset \\
    \displaystyle
  \argmin_{i \in A_k} \hat{\mu}^k_2(i), & 
  \{i \in A_k \colon \hat{\mu}^k_2(i) \leq \tau\} = \emptyset.
  \end{array}
  \right.
\end{equation}
Then, the gaps $\delta(\hat{J}(A_k),i)$ are estimated for all arms~$i
\in A_k$ as follows.
\begin{equation}
  \hat\delta\left(\hat{J}(A_k),i\right) :=
  \delta\left(\hat{\mu}^k(\hat{J}(A_k),\hat{\mu}^k(i)\right),
  \label{eq:delta_hat}
\end{equation}
where $\hat{\mu}^k(i) :=
\left(\hat{\mu}^k_1(i),\hat{\mu}^k_2(i)\right).$ In other words, the
gaps relative to the `seemingly optimal' arm are estimated by
replacing the (unknown) arm attributes by their available estimates.
\ignore{
If $\hat{\mu}^k_{2}(\hat{J}(A_k)) \leq \tau,$
i.e., arm~$\hat{J}(A_k)$ `appears' feasible, then
\begin{equation*}
    \displaystyle
  \hat\delta(\hat{J}(A_k),i) = \left\{
  \begin{array}{ll}
    \sqrt{a_1} \left ( \hat{\mu}^k_{1}(i)- \hat{\mu}^k_{1}(\hat{J}(A_k)) \right) & \text{ if }  \hat{\mu}^k_{2}(i) \leq \tau \text{ (i.e., $i$ appears feasible)} \\
    \sqrt{a_2} \left (\hat{\mu}^k_2(i)-\tau \right) & \text{ if } \hat{\mu}^k_{2}(i) > \tau,\ \hat{\mu}^k_1(i) \leq \hat{\mu}^k_{1}(\hat{J}(A_k)) \\
    & \text{\quad (i.e., $i$ appears to be a deceiver)} \\
    \max \{ \sqrt{a_2} \left ((\mu^k_2(i) - \tau \right ) , \sqrt{a_1} \left  (\mu_1^k(i)-\mu_1^k(\hat{J}(A_k)) \right ) \}
    & \text{ if } \hat{\mu}^k_{2}(i) > \tau,\ \hat{\mu}^k_1(i) > \hat{\mu}^k_{1}(\hat{J}(A_k))\\    
  \end{array}
  \right..
\end{equation*}
}% end ignore

Finally, the arm $\argmax_{i \in A_k} \hat{\delta}(\hat{J}(A_k),i),$
i.e., the arm with the largest estimated gap relative to
$\hat{J}(A_k),$ is rejected, with the following rule used to break
ties.\footnote{This tie-breaking rule plays a key role in the
performance of \algo; in contrast, the tie-breaking rule is
\emph{inconsequential} in the original SR algorithm for single
attribute MABs.}
%Note that \algo uses a sophisticated
%tie-breaking rule, as opposed to classical SR, which uses a random
%tie-breaking rule. We first compute
Let
\begin{equation}
\hat{D}(A_k) = \{ \underset{i \in A_k, i \neq \hat{J}(A_k) }{\argmax} \hat{\Delta}(\hat{J}(A_k),i) \}
\end{equation}
denote the set of arms in $A_k$ that achieve the same maximizing
$\hat\Delta(\hat{J}(A_k),\cdot)$. We denote by
$\mathcal{K}(\hat{D}(A_k))$ the set of arms in $\hat{D}(A_k)$ that
appear empirically feasible (i.e., satisfying~$\hat{\mu}^k_2(\cdot)
\leq \tau$).

\noindent $\bullet$ If $\mathcal{K}(\hat{D}(A_k))$ is a strict subset
of $\hat{D}(A_k),$ the arm that appears the most infeasible (i.e., the
arm in $\hat{D}(A_k) \setminus \mathcal{K}(\hat{D}(A_k))$ with the
largest value of $\hat\mu_2^k(\cdot)$) is rejected.

\noindent $\bullet$ Else, the arm that appears feasible, but most
suboptimal (i.e., the arm in $\mathcal{K}(\hat{D}(A_k))$ with the
largest value of $\hat\mu_1^k(\cdot)$) is rejected.

%all arms in $\hat{D}(A_k)$ are empirically feasible, we reject the
%one with the highest $\hat\mu_1^k(\cdot)$. Else, the empirically
%infeasible arm with the highest $\hat\mu_2^k(\cdot)$ is rejected.
%Finally, at the end of phase~$K-1,$ the algorithm outputs the sole
%surviving arm and a boolean feasibility flag denoting whether the
%instance is feasible or not. The feasibility flag is set to
%\texttt{True} if the sole surviving arm is feasible and
%\texttt{False} otherwise.

\emph{Remark:} We motivate the rationale behind the tie-breaking rule
of the \algo\ algorithm via the scenario shown in Figure~\ref{fig:
  algo_tie_breaking} at the end of a generic phase. Here, arm~1
appears optimal, with $\hat\Delta(1,2) = \hat\Delta(1,3),$ both gaps
being equal to the (small) feasibility gap of arm~1 (i.e., $\tau -
\hat\mu_2(1)$). However, the arms~2 and~3 are not `symmetric' from the
standpoint of the algorithm. Since the feasibility/infeasibility
status of arm~1 is `uncertain' (given how close it is to the $\tau$
boundary), eliminating arm~3 is riskier, since it might be the optimal
arm in case arm~1 is subsequenly found to be infeasible. On the other
hand, eliminating arm~2 first is `safer,' since it is less likely to
be the optimal arm.

\ignore{Unlike the classical SR algorithm, these two arms having the
  same value of $\Delta(1,\cdot)$ is not a zero probability
  event. Moreover, these two arms are not `symmetric' for a random
  tie-breaking rule to make sense; arm 2 is infeasible while arm 3 is
  feasible. Recommending an infeasible arm at the end would result in
  the algorithm predicting both the optimal arm and the feasibility
  flag wrong, while recommending a feasible arm at the end would
  result in an error in predicting only the optimal arm. Thus, it
  makes sense to `hoard' feasible arms and remove infeasible arms
  first, which is the motivation behind the Infeasible First algorithm
  (see Appendix~\ref{app: IF} for a formal description of the
  algorithm).}

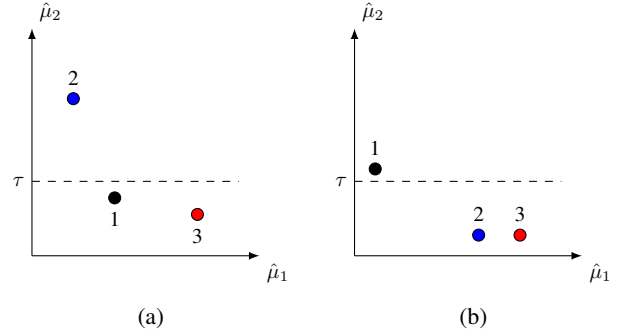
\begin{figure}[t]
  \centering
     \begin{subfigure}{0.49\linewidth}
         \centering
  \begin{tikzpicture}[scale=0.55,    
  dot/.style = {circle, draw, fill=#1, inner sep=2pt},every node/.style={scale=0.8}
  ]
    % axes
    \draw[draw,latex-] (0,5) +(0,0.5cm) node[above right] {$\hat\mu_2$} -- (0,0);
    \draw[draw,-latex] (0,0) -- (5,0) -- +(0.5cm,0) node[below right] {$\hat\mu_1$};
    
    % tau line
    \draw[dashed] (0,1.8) node[left] {$\tau$} -- (5,1.8);
    
    % points 
    \node[dot=black,  label=below:{1}]  at (2, 1.4) {};
    \node[dot=blue,   label=2]  at (1, 3.8) {};
    \node[dot=red,    label=below:{3}]  at (4, 1) {};
  \end{tikzpicture}

         \caption{}
         \label{fig: algo_tie_breaking}
     \end{subfigure}
     \hfill
     \begin{subfigure}{0.49\linewidth}
         \centering
  \begin{tikzpicture}[scale=0.55,    
  dot/.style = {circle, draw, fill=#1, inner sep=2pt},every node/.style={scale=0.8}
  ]
    % axes
    \draw[draw,latex-] (0,5) +(0,0.5cm) node[above right] {$\hat\mu_2$} -- (0,0);
    \draw[draw,-latex] (0,0) -- (5,0) -- +(0.5cm,0) node[below right] {$\hat\mu_1$};
    
    % tau line
    \draw[dashed] (0,1.8) node[left] {$\tau$} -- (5,1.8);
    
    % points 
    \node[dot=black,  label=1]  at (0.5, 2.1) {};
    \node[dot=blue,   label=2]  at (3, 0.5) {};
    \node[dot=red,    label=3]  at (4, 0.5) {};
  \end{tikzpicture}
         \caption{}
         \label{fig: algo_if}
     \end{subfigure}
     \caption{Panel~(a) shows a feasible instance that motivates our
       tie-breaking rule. Panel~(b) shows a feasible instance that
       motivates the use of estimates of our information theoretic
       suboptimality gaps to guide arm elimination.}
        \label{fig: algo_motivation}
\end{figure}

\emph{Remark:} While the above example might suggest that it is sound
to blindly eliminate seemingly infeasible arms first, the scenario
shown in Figure~\ref{fig: algo_if} (again, at the end of a generic
phase) demonstrates that this is not always the case. Here, arm~2
appears optimal, but arm~1, placed slightly above the $\tau$ boundary,
might be optimal if $\hat\mu_2(1)$ is a (small) overestimation of
$\mu_2(1)$. It is therefore `safer' in this scenario to eliminate
arm~3; this is exactly what \algo\ would do, since $\hat\Delta(2,1) <
\hat\Delta(2,3).$ This highlights the importance of the sophisticated
elimination criterion employed by \algo, that captures the relative
likelihoods of different arms being optimal (via estimates of
information theoretic suboptimality gaps).

%end ignore

  \begin{algorithm}[tb]
  \caption{\algo\ algorithm}
   \label{algo: pairwise}
  \begin{algorithmic}[1]
    \Procedure{C-SR}{$T,K,\tau$}
    \State Let $A_1=\{1,\ldots,K\}$
    \State $\overline{\log}(K) := \frac{1}{2} + \sum_{i=2}^{K}\frac{1}{i}$ 
    \State $n_0 = 0,$ $n_k = \lceil \frac{1}{\overline{\log}(K) } \frac{T-K}{K+1-k} \rceil$ for $1 \leq k \leq K-1$
    %\For{$k=1,\ldots,K-1$} 
    %  \State  $n_k= \lceil \frac{1}{\overline{\log} K } \frac{n-K}{K+1-k} \rceil$.
    %  \EndFor
      \For{$k=1,\ldots,K-1$}
      	\State For each $i \in A_k$, pull arm $i$ $(n_k-n_{k-1})$ times
	\State Compute~$\hat{J}(A_k)$ (using~\eqref{eq:emp_opt_arm})
	\State Compute~$\hat\Delta\left(\hat{J}(A_k),i\right)$ for $i \in A_k$ (using~\eqref{eq:delta_hat})
	\State $ \hat{D}(A_k) = \{ \underset{i \in A_k, i \neq \hat{J}(A_k) }{\argmax} \hat{\Delta}(\hat{J}(A_k),i) \}$
	\If{$| \hat{D}(A_k) | <1$}
		\State $A_{k+1}=A_k \setminus  \hat{D}(A_k)$ 
	\Else 
		\State Compute $\mathcal{K}(\hat{D}(A_k))$
		\If{$\mathcal{K}(\hat{D}(A_k))^{\mathsf{c}} = \emptyset $}
			\State $A_{k+1}=A_k \setminus \{ \underset{i \in\mathcal{K}(\hat{D}(A_k)) }{\argmax} \hat\mu_1^k(i) \}$ 
		\Else 
			\State $A_{k+1}=A_k \setminus \{ \underset{i \in\mathcal{K}(\hat{D}(A_k))^{\mathsf{c}} }{\argmax} \hat\mu_2^k(i) \}$ 
		\EndIf 
        \EndIf 
     \EndFor
      \State Let $\hat{J}([K])$ be the unique element of $A_K$
      \If{$\hat{\mu}^{K-1}_{2}(\hat{J}_{T})> \tau$}
      	\State $\hat{F}([K])=$\texttt{False}
     \Else
     	\State $\hat{F}([K])=$\texttt{True}
      \EndIf
      \State \textbf{return} ($\hat{J}([K]),\hat{F}([K])$)
    \EndProcedure
  \end{algorithmic}
\end{algorithm}
  \noindent {\bf Performance evaluation:} We now characterize the
  performance of \algo. For the purpose of expressing our performance
  guarantee, we order the arm labels as follows (without loss of
  generality). Arm~1 is the optimal arm, and arms $2,\ldots,K$ are
  labelled in increasing order of $\Delta(1,\cdot),$ with ties broken
  in a manner that is consistent with the \algo\  algorithm. Formally,
  for any $1 < i < j \leq K,$
  %we either have $\Delta(1,i) < \Delta(i,j),$ or we have
  %
  if $\Delta(1,i) = \Delta(1,j),$ then either\\
  \noindent $\bullet$ $i,j \in \mathcal{K}(\nu)$ and $\mu_1(i) \leq \mu_1(j),$ or \\
  \noindent $\bullet$ $i \in \mathcal{K}(\nu)$ and $j \notin \mathcal{K}(\nu),$ or \\
  \noindent $\bullet$ $i,j \notin \mathcal{K}(\nu)$ and $\mu_2(i) \leq \mu_2(j).$\\ 

   \begin{theorem}
    \label{thm: ub}
    Under the \algo\  algorithm, the probability of error is upper bounded
    as: 
    $$e_T \leq c(K) \exp \left ( -\frac{\beta
      T}{H_2\ \overline{\log}(K)} \right ),$$ where $H_2=\underset{i
      \in [K], i \neq 1}{\max} \frac{i}{\Delta^2(1,i)},$ $c(K)$ is a
    function of~$K,$ and $\beta$ is a positive universal constant.
  \end{theorem}

  The main takeaways from Theorem~\ref{thm: ub} are as follows.

  \noindent $\bullet$ Theorem~\ref{thm: ub} provides an upper bound on
  the probability of error under \algo, that decays exponentially with
  the budget~$T.$ The associated decay rate is given by
  $\frac{\beta}{H_2\ \overline{\log}(K)},$ suggesting that the
  instance-dependent parameter $H_2$ captures the hardness of the
  instance (under the \algo\  algorithm); a larger value of $H_2$
  implies a `harder' instance, since the probability of error decays
  more slowly with the budget.

  \noindent $\bullet$ The `hardness index' $H_2$ agrees with the
  hardness index obtained for the classical SR algorithm
  in~\cite{audibert-bubeck} (also denoted $H_2$) for the unconstrained
  MAB problem when~$\tau \ra \infty.$

  \noindent $\bullet$ The decay rate from the upper bound for \algo\ 
  can be compared with that in the information theoretic lower bound
  conjectured in Section~\ref{sec: lb} (see Conjecture~\ref{thm: K
    arms lb}). Indeed, it can be proved that $\frac{H_2}{2} \leq H_1
  \leq \overline{\log}(K) H_2$ (see \cite{audibert-bubeck}). This
  suggests that the decay rate under \algo\  is optimal up to a factor
  that is logarithmic in the number of arms. In other words, this
  suggests \algo\  is nearly optimal.\footnote{The same logarithmic (in
  the number of arms) `gap' between the decay rate in the information
  theoretic lower bounds and that of the best known upper bound also
  exists in the (unconstrained, fixed budget) pure exploration MAB
  problem (see~\cite{audibert-bubeck,kaufmann16}).}

  \noindent {\bf Sketch of the proof of Theorem~\ref{thm: ub}:} In the
  remainder of this section, we sketch the proof of Theorem~\ref{thm:
    ub}. The
  complete proof can be found in Appendix~\ref{app: ub}. Note that
\begin{align*}
  e_T &=\mathbb{P} \left (  \left \{ J([K]) \neq  \hat{J}([K]) \right \} \cup \left \{ (\hat{F}([K]) \neq F([K]) \right \} \right ) \\
  &= \sum_{k=1}^{K-1} \prob{\textrm{Arm
      1 is dismissed in round }k} \\ 
      &+
  \prob{\hat{J}([K])=J([K]),\hat{F}([K])\neq F([K])}.%\\
  %&=:\sum_{k=1}^{K-1} \prob{\mathcal{A}_k} + \prob{\hat{J}_{T}=1,\hat{O}([K])=0}%\\
\end{align*}
Let $\mathcal{A}_k$ denote the event that arm~1 is rejected at the end
of round~$k.$ Noting that the event in the last term above implies
that the feasibility status of~arm~1 is estimated incorrectly at the
end of phase~$K-1,$ \eqref{eqn: conc_inequality} implies
\begin{align}
  e_T &\leq \sum_{k=1}^{K-1} \prob{\mathcal{A}_k} + 2 \exp\left(-a_2 n_{K-1} ( | \tau - \mu_2(1) |)^2 \right) \\
  &\leq \sum_{k=1}^{K-1} \prob{\mathcal{A}_k} + 2 \exp\left(-n_{K-1} \Delta^2(1,2) \right).
 \label{eq:error1-CSR}
\end{align}
We now bound $\prob{\mathcal{A}_k}.$ In round $k$, at least one of the
$k$ `worst' arms (according to the ordering defined on the arms)
survives (i.e., belongs to $A_k$). Thus, for arm 1 to be dismissed at
the end of round $k,$ it must appear empirically `worse' than this
arm. Formally, we have
\begin{align*}
  \prob{\mathcal{A}_k} &\leq \sum_{j=K-k+1}^{K} \prob{\hat{J}(A_k) = j}
 \\ &+ \sum_{i=2}^{K-k}\sum_{j=K-k+1}^{K} \prob{\hat{J}(A_k) = i, \hat{\delta}_k(i,j) \leq \hat{\delta}_k(i,1)} \\
  &=: S_1 + S_2.
\end{align*}
The summation~$S_1$ above corresponds to the event that one of the
worst $k$ arms looks empirically optimal at the end of phase~$k.$ On
the other hand, the summation~$S_2$ corresponds to the event that some
other arm~$i$ (not among the worst~$k$ arms) looks empirically optimal
at the end of phase~$k,$ and further that arm~1 has a greater
(estimated) gap (relative to~$i$) than an arm~$j,$ which is among the
worst~$k$ arms (this is necessary for the elimination of
arm~1.). Crucially, the terms in~$S_1$ can be bounded by analysing a
\emph{two-armed} instance consisting only of arms~1 and~$j.$
Similarly, the terms in~$S_3$ can be bounded by analysing a
\emph{three-armed} instance consisting only of arms~1,~$i$ and~$j.$
The relevant bounds are summarized below.
\begin{lemma}
  \label{lemma:2arm-upper_bound}
  Consider a two-armed instance where the arms are labelled (without
  loss of generality) as per the convention described before. Under
  \algo, the probability that arm~2 is optimal after phase~1 is at
  most $c_2 \exp\left(-\beta_2 n_1 \Delta^2(1,2)\right),$ where
  $c_2,$ $\beta_2$ are universal positive constants.
\end{lemma}

\begin{lemma}
  \label{lemma:3arm-upper_bound}
  Consider a three-armed instance where the arms are labelled (without
  loss of generality) as per the convention described before. Under
  \algo, the probability that after phase~1, arm 2 is empirically optimal and arm~1 is rejected is at
  most $c_3 \exp\left(-\beta_3 n_1 \Delta^2(1,3)\right),$ where $c_3,$
  $\beta_3$ are universal positive constants.
\end{lemma}
Using Lemmas~\ref{lemma:2arm-upper_bound}
and~\ref{lemma:3arm-upper_bound} (proofs in Appendix~\ref{app: ub}),
$\prob{\mathcal{A}_k}$ can be upper bounded as follows:
\begin{align*}
  \prob{\mathcal{A}_k} &\leq k c_2 \exp\left(-\beta_2 n_k
  \Delta^2(1,K-k+1) \right) \\&+ k(K-k-1) c_3 \exp\left(-\beta_3 n_k
  \Delta^2(1,K-k+1) \right) \\
  &\leq K^2 \ \tilde{c} \exp\left(-\tilde{\beta} n_k \Delta^2(1,K-k+1) \right),
\end{align*}
where $\tilde{c} = \max(c_2,c_3)$ and $\tilde\beta =
\min(\beta_2,\beta_3).$ Finally, substituting the above bound
into~\eqref{eq:error1-CSR}, we get
\begin{align*}
  e_T &\leq \sum_{k=1}^{K-1} K^2 \ \tilde{c} \exp\left(-\tilde{\beta} n_k \Delta^2(1,K-k+1) \right) \\
  &+ 2 \exp\left(- n_{K-1} \Delta^2(1,2) \right) \\
  & \leq (K^3 \tilde{c} + 2) \exp\left(-\hat\beta \min_{1 \leq k \leq K-1} \left(n_k \Delta^2(K-k+1) \right)\right),
\end{align*}
where $\hat\beta = \min(\tilde\beta,1).$ Now, using the definition
of $n_k,$ 
\begin{align*}
&\min_{1 \leq k \leq K-1} \left(n_k \Delta^2(1,K-k+1) \right) \\
&\quad \geq \min_{1 \leq k \leq K-1} \left( \frac{T-K}{\overline{\log}(K)}
\frac{\Delta^2(K-k+1)}{K-k+1}\right) \geq
\frac{T-K}{\overline{\log}(K)} \frac{1}{H_2},
\end{align*} 
which implies the
statement of the theorem.

%end ignore
%%%%%%%%%%%%%%%%%%%%%%%%%%%%%%%%%%%%%%%%%%%%%%

%% file: sections/numerics.tex
\section{Numerical experiments}
\label{sec: numerics}
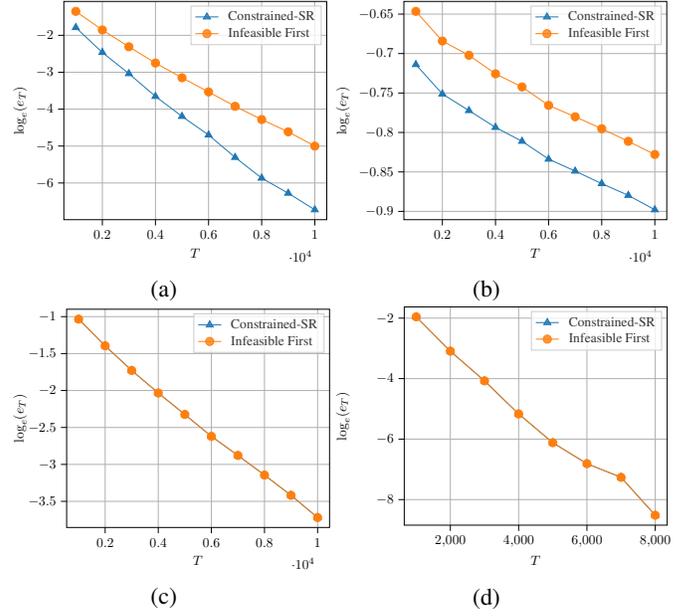
\begin{figure}[t]
     \centering
     \begin{subfigure}{0.49\linewidth}
         \centering
 	\scalebox{0.85}{\input{pics/inst1.tex}}
          \caption{}
         \label{fig: csr_better}
     \end{subfigure}
     \hfill%
     \begin{subfigure}{0.49\linewidth}
         \centering
  	\scalebox{0.85}{\input{pics/inst3.tex}}

         \caption{}
         \label{fig: csr_factor}
     \end{subfigure}
     
     \begin{subfigure}{0.49\linewidth}
         \centering
 	\scalebox{0.85}{\input{pics/inst4.tex}}
          \caption{}
         \label{fig: both_same}
     \end{subfigure}
     \hfill %
     \begin{subfigure}{0.49\linewidth}
         \centering
  	\scalebox{0.85}{\input{pics/inst2.tex}}

         \caption{}
         \label{fig: sim_infeasible}
     \end{subfigure}
        \caption{The numerical performance of IF and \algo\  is shown on three different feasible instances in panels~(a), (b), (c) and on an infeasible instance in panel~(d) . Note that the probability of error decays exponentially with the horizon in all four cases. Panels (b) and (c) show that IF and \algo\  have similar performance on those instances, while in Panel (a), the decay rate of \algo\  is higher.
        }
        \label{fig: sims}
\end{figure}
In this section, we present the results of simulations that show the
performance of the \algo\ algorithm. We consider 2-dimensional jointly
Gaussian arms with the covariance matrix $\begin{bmatrix} 1 & 0.5
  \\ 0.5 & 1
\end{bmatrix}$ and attributes as defined in
Section~\ref{sec: lb}. We compare the performance of \algo\ with that
of Infeasible First (IF), which also follows a Successive Rejects
based framework but differs from \algo\ in the way arms are
rejected. In round~$k$, IF removes the arm with the highest empirical
constraint attribute (i.e., the most infeasible looking arm) if $A_k$
contains infeasible looking arms, and otherwise removes the arm with
the highest empirical objective attribute (i.e., the arm that looks
like the most suboptimal feasible arm). See Appendix~\ref{app: IF} for
a formal description of this algorithm.

In the first instance, the mean vectors of the arms are $[1\ 0.95]^T,$
$[5\ 0.001]^T,$ and $[10\ 0.001]^T$. The threshold $\tau$, which is
the upper bound for the mean of the second dimension, is fixed
at~1. Thus, arm~1 is optimal, and arms~2 and~3 are feasible
suboptimal. This instance is motivated by the scenario described in
Figure~\ref{fig: algo_if}. The second instance that we consider is
feasible and has three arms with the mean vectors $[1\ 0.995]^T$,
$[2\ 1.005]^T$ and $[12\ 0.001]^T$ with $\tau=1$. Thus, arm 1 is
optimal, arm 2 is a deceiver and arm 3 is feasible suboptimal. The
third instance that we consider is also feasible and has four arms
with the mean vectors $[0.3\ 0.45]^T$, $[0.35\ 0.45]^T$,
$[0.2\ 0.8]^T$ and $[0.5\ 0.8]^T$ and $\tau=0.5$. Thus, arm 1 is
optimal, arm 2 is feasible suboptimal, arm 3 is a deceiver and arm 4
is infeasible suboptimal. The fourth instance that we consider is
infeasible and has four arms with the mean vectors $[0.3\ 1.6]^T,$
$[0.4\ 1.7]^T,$ $[0.2\ 1.1]^T,$ and $[0.5\ 1.2]^T$ and $\tau=1$. Thus,
arm 3 is the optimal arm for this instance. The results of the
simulations for each of these instances can be found in
Figures~\ref{fig: csr_better}, \ref{fig: csr_factor}, \ref{fig:
  both_same} and \ref{fig: sim_infeasible} respectively.

The algorithms were run for horizons up to 10000 and averaged over
100000 runs for the feasible instances and over 10000 runs for the
infeasible instance.  Empirical averages were used as the attribute
estimators. Figure~\ref{fig: sims} shows the variation of
$\log_{e}(e_T)$ with the horizon $T$ for these four instances. Note
that the slope of this curve captures the (exponential) decay rate of the
probability of error.
%As expected, the probability of error decays exponentially with the
%horizon for both algorithms.
In the case of the infeasible instance (Figure~\ref{fig:
  sim_infeasible}), the performance is nearly the same. In
Figures~\ref{fig: csr_factor} and~\ref{fig: both_same}, we once again
observe that the decay rates of \algo\ and IF are identical; the
probability of error under \algo\ appears to be smaller than that
under IF by a constant factor in Figure~\ref{fig:
  csr_factor}. However, in Figure~\ref{fig: csr_better},
\algo\ demonstrates a superior decay rate, since it employs a more
sophisticated elimination criterion using gaps inspired by the
two-armed lower bound (as noted in Section~\ref{sec: algo}).

%% file: pics/inst1.tex
% This file was created with tikzplotlib v0.10.1.
\begin{tikzpicture}[scale=0.6]

\definecolor{darkgray176}{RGB}{176,176,176}
\definecolor{darkorange25512714}{RGB}{255,127,14}
\definecolor{lightgray204}{RGB}{204,204,204}
\definecolor{steelblue31119180}{RGB}{31,119,180}

\begin{axis}[
legend cell align={left},
legend style={fill opacity=0.8, draw opacity=1, text opacity=1, draw=lightgray204},
tick align=outside,
tick pos=left,
x grid style={darkgray176},
xlabel={$T$},
xmajorgrids,
xmin=550, xmax=10450,
xtick style={color=black},
y grid style={darkgray176},
ylabel={\(\displaystyle \log_{e}(e_T)\)},
ymajorgrids,
ymin=-6.99416097789813, ymax=-1.0821613522794,
ytick style={color=black}
]
\addplot [semithick, steelblue31119180, mark=triangle*, mark size=3, mark options={solid}]
table {%
1000 -1.78958184198028
2000 -2.46392824340281
3000 -3.0380136658106
4000 -3.6546712827842
5000 -4.19571305661039
6000 -4.69948086545933
7000 -5.3063495382453
8000 -5.86747856732699
9000 -6.27648350214028
10000 -6.72543372218818
};
\addlegendentry{Constrained-SR}
\addplot [semithick, darkorange25512714, mark=*, mark size=3, mark options={solid}]
table {%
1000 -1.35088860798934
2000 -1.85489098066817
3000 -2.30991186838051
4000 -2.75027943531995
5000 -3.14888345304817
6000 -3.53324416258249
7000 -3.92308395278757
8000 -4.28163846064261
9000 -4.61421093064024
10000 -5.0011801353255
};
\addlegendentry{Infeasible First}
\end{axis}

\end{tikzpicture}

%% file: pics/inst3.tex
% This file was created with tikzplotlib v0.10.1.
\begin{tikzpicture}[scale=0.6]

\definecolor{darkgray176}{RGB}{176,176,176}
\definecolor{darkorange25512714}{RGB}{255,127,14}
\definecolor{lightgray204}{RGB}{204,204,204}
\definecolor{steelblue31119180}{RGB}{31,119,180}

\begin{axis}[
legend cell align={left},
legend style={fill opacity=0.8, draw opacity=1, text opacity=1, draw=lightgray204},
tick align=outside,
tick pos=left,
x grid style={darkgray176},
xlabel={$T$},
xmajorgrids,
xmin=550, xmax=10450,
xtick style={color=black},
y grid style={darkgray176},
ylabel={\(\displaystyle \log_{e}(e_T)\)},
ymajorgrids,
ymin=-0.910532178096705, ymax=-0.633939314136687,
ytick style={color=black}
]
\addplot [semithick, steelblue31119180, mark=triangle*, mark size=3, mark options={solid}]
table {%
1000 -0.713921479776422
2000 -0.751221307654977
3000 -0.772233678880729
4000 -0.793564379012418
5000 -0.811142738797653
6000 -0.833766490369547
7000 -0.848935867855923
8000 -0.864813704366568
9000 -0.87966954841851
10000 -0.897959775189432
};
\addlegendentry{Constrained-SR}
\addplot [semithick, darkorange25512714, mark=*, mark size=3, mark options={solid}]
table {%
1000 -0.646511717043961
2000 -0.684127976139785
3000 -0.702167743781028
4000 -0.725670372265505
5000 -0.742295408826676
6000 -0.765653359346824
7000 -0.78010037742389
8000 -0.7952463535154
9000 -0.811210255423648
10000 -0.827867851524258
};
\addlegendentry{Infeasible First}
\end{axis}

\end{tikzpicture}

%% file: pics/inst4.tex
% This file was created with tikzplotlib v0.10.1.
\begin{tikzpicture}[scale=0.6]

\definecolor{darkgray176}{RGB}{176,176,176}
\definecolor{darkorange25512714}{RGB}{255,127,14}
\definecolor{lightgray204}{RGB}{204,204,204}
\definecolor{steelblue31119180}{RGB}{31,119,180}

\begin{axis}[
legend cell align={left},
legend style={fill opacity=0.8, draw opacity=1, text opacity=1, draw=lightgray204},
tick align=outside,
tick pos=left,
x grid style={darkgray176},
xlabel={$T$},
xmajorgrids,
xmin=550, xmax=10450,
xtick style={color=black},
y grid style={darkgray176},
ylabel={\(\displaystyle \log_{e}(e_T)\)},
ymajorgrids,
ylabel shift=-0.2cm,
ymin=-3.85360253129814, ymax=-0.899797397909164,
ytick style={color=black}
]
\addplot [semithick, steelblue31119180, mark=triangle*, mark size=3, mark options={solid}]
table {%
1000 -1.03406126760866
2000 -1.39517366572649
3000 -1.72833406737796
4000 -2.03355081491475
5000 -2.32401304040767
6000 -2.62031408679812
7000 -2.87883852208249
8000 -3.14446432735439
9000 -3.41854701999727
10000 -3.71933866159864
};
\addlegendentry{Constrained-SR}
\addplot [semithick, darkorange25512714, mark=*, mark size=3, mark options={solid}]
table {%
1000 -1.03406126760866
2000 -1.39517366572649
3000 -1.72833406737796
4000 -2.03355081491475
5000 -2.32401304040767
6000 -2.62031408679812
7000 -2.87883852208249
8000 -3.14446432735439
9000 -3.41854701999727
10000 -3.71933866159864
};
\addlegendentry{Infeasible First}
\end{axis}

\end{tikzpicture}

%% file: pics/inst2.tex
% This file was created with tikzplotlib v0.10.1.
\begin{tikzpicture}[scale=0.6]

\definecolor{darkgray176}{RGB}{176,176,176}
\definecolor{darkorange25512714}{RGB}{255,127,14}
\definecolor{lightgray204}{RGB}{204,204,204}
\definecolor{steelblue31119180}{RGB}{31,119,180}

\begin{axis}[
legend cell align={left},
legend style={fill opacity=0.8, draw opacity=1, text opacity=1, draw=lightgray204},
tick align=outside,
tick pos=left,
unbounded coords=jump,
x grid style={darkgray176},
xlabel={$T$},
xmajorgrids,
xmin=650, xmax=8350,
xtick style={color=black},
y grid style={darkgray176},
ylabel={\(\displaystyle \log_{e}(e_T)\)},
ymajorgrids,
ymin=-8.84520935152247, ymax=-1.62885382918543,
ytick style={color=black},
every axis y label/.style={
   at={(-0.19,0.5)},rotate=90,anchor=near ticklabel}
]
\addplot [semithick, steelblue31119180, mark=triangle*, mark size=3, mark options={solid}]
table {%
1000 -1.95686998929166
2000 -3.09224317393483
3000 -4.07454193492592
4000 -5.16728910414163
5000 -6.11929791861787
6000 -6.81244509917781
7000 -7.26443022292087
8000 -8.51719319141624
9000 -inf
10000 -inf
};
\addlegendentry{Constrained-SR}
\addplot [semithick, darkorange25512714, mark=*, mark size=3, mark options={solid}]
table {%
1000 -1.95686998929166
2000 -3.09224317393483
3000 -4.07454193492592
4000 -5.16728910414163
5000 -6.11929791861787
6000 -6.81244509917781
7000 -7.26443022292087
8000 -8.51719319141624
9000 -inf
10000 -inf
};
\addlegendentry{Infeasible First}
\end{axis}

\end{tikzpicture}

%% file: sections/conclusion.tex
\section{Concluding Remarks}
\label{sec:conclusion}

This work motivates follow-ups in several directions. On the
theoretical front, the main gap in this work pertains to the
information theoretic lower bound. Proving Conjecture~\ref{thm: K arms
  lb} would not only establish the `near' optimality of the
\algo\ algorithm, but also, quite likely, introduce a novel approach
for deriving lower bounds in the fixed budget pure exploration
setting. On the application front, the present work motivates an
extensive case study applying the proposed algorithm in various
application scenarios.

This work also motivates generalizations to constrained reinforcement
learning, where the goal is to identify the optimal policy that
fulfills additional constraints.

%% file: sections/appendix_A.tex
\section{Proof of Theorem \ref{thm: 2 arms lb}} \label{app: 2_arm_lb_proof}

\begin{proof}
With some abuse of notation, we denote by $(J(\nu),F(\nu))$ the
correct output for instance $\nu$. Consider any alternative bandit
model $\nu^{\prime}=\left(\nu^{\prime}(1), \nu^{\prime}(2)\right)$
such that its correct output, $(J(\nu^{\prime}), F(\nu^{\prime})) \neq
(J(\nu),F(\nu))$. Let $\mathcal{L}$ be a consistent algorithm. We
apply Lemma 1 of \cite{kaufmann16} with the stopping time $\sigma=T$
a.s. on the event $\mathcal{H}= \{ \hat{J}([2]) = J(\nu) \} \cap \{
\hat{F}([2]) = F(\nu) \} $ to get:
\begin{align*}
\mathbb{E}_{\nu^{\prime}}\left[N_{1}(T)\right] \textrm{KL}\left(\nu^{\prime}(1), \nu(1)\right)
&+\mathbb{E}_{\nu^{\prime}}\left[N_{2}(T)\right] \textrm{KL}\left(\nu^{\prime}(2), \nu(2)\right) \\
&\geq d\left(\mathbb{P}_{\nu^{\prime}}(\mathcal{H}), \mathbb{P}_{\nu}(\mathcal{H})\right), \numberthis \label{eqn: kaufmann_lemma1}
\end{align*}
where $\mathbb{E}_{\nu}(\cdot)$ and $\mathbb{P}_{\nu}(\cdot)$ denote
the expectation and the probability, respectively, with respect to the
randomness introduced by the interaction of the algorithm with the
bandit instance $\nu$, and $d(\cdot, \cdot)$ denotes the binary
relative entropy. Denote by $e_T(\nu)$ the probability of error of the
algorithm on the instance $\nu$.

We have that $e_{T}(\nu)=1-\mathbb{P}_{\nu}(\mathcal{H})$ and
$e_{T}\left(\nu^{\prime}\right) \geq
\mathbb{P}_{\nu^{\prime}}(\mathcal{H})$. As algorithm $\mathcal{L}$ is
consistent, we have that for every $\epsilon>0, \exists
T_{0}(\epsilon)$ such that for all $T \geq T_{0}(\epsilon),
\mathbb{P}_{\nu^{\prime}}(\mathcal{H}) \leq \epsilon \leq
\mathbb{P}_{\nu}(\mathcal{H})$. For $T \geq T_{0}(\epsilon)$, we have:
\begin{align*}
&\mathbb{E}_{\nu^{\prime}} \left[N_{1}(T)\right] \textrm{KL}\left(\nu^{\prime}(1), \nu(1)\right)
+\mathbb{E}_{\nu^{\prime}}\left[N_{2}(T)\right] \textrm{KL}\left(\nu^{\prime}(2), \nu(2)\right) \\
&\quad \geq d\left(\epsilon, 1-e_{T}(\nu)\right) \geq(1-\epsilon) \log \frac{1-\epsilon}{e_{T}(\nu)}+\epsilon \log \epsilon. 
\end{align*}
In the limit where $\epsilon$ goes to zero, we have,
\begin{align*}
&\limsup_{T \rightarrow \infty}-\frac{1}{T} \log e_{T}(\nu) \\
&\leq \limsup_{T \rightarrow \infty} \sum_{j=1}^{2} \frac{\mathbb{E}_{\nu^{\prime}}\left[N_{j}(T)\right]}{T} \textrm{KL}\left(\nu^{\prime}(j), \nu(j)\right) \\
&\leq \max_{j=1,2} \textrm{KL}\left(\nu^{\prime}(j), \nu(j)\right).
\end{align*}
Denote by $\mathcal{M}$ the set of two-armed bandit instances whose
arms belong to $\mathcal{D}$. Minimizing the RHS over all
$\nu^{\prime} \in \mathcal{M}$ whose correct output differs from that
of $\nu$ gives us:
\begin{align*}
&\limsup_{T \rightarrow \infty}-\frac{1}{T} \log e_{T}(\nu) \\
&\leq \underset{ \substack{\nu^{\prime} \in \mathcal{M}: \\ J(\nu^{\prime} ) \neq J(\nu) \text{ or} \\  F(\nu^{\prime} ) \neq F(\nu) } }{\inf} \max_{j=1,2} \textrm{KL}\left(\nu^{\prime}(j), \nu(j)\right). \numberthis \label{eqn: kaufmann}
\end{align*}
Using the formula for the KL divergence between two multivariate
distributions in \eqref{eqn: kaufmann} gives:
\begin{align*}
&\limsup_{T \rightarrow \infty}-\frac{1}{T} \log e_{T}(\nu) \\
&\leq \frac{1}{2} \underset{ \substack{\nu^{\prime} \in \mathcal{M}: \\ J(\nu^{\prime} ) \neq J(\nu) \text{ or} \\  F(\nu^{\prime} ) \neq F(\nu) } }{\inf} \max \Bigl \{  a_1\left(\mu_{1}(1)-\mu_{1}^{\prime}(1)\right)^2 \\ &+  a_2\left(\mu_{2}(1)-\mu_{2}^{\prime}(1)\right)^2  ,  
a_1\left(\mu_{1}(2)-\mu_{1}^{\prime}(2)\right)^2 \\
&+  a_2\left(\mu_{2}(2)-\mu_{2}^{\prime}(2)\right)^2  \Bigr \}. \numberthis \label{eqn: gaps_lb}
\end{align*}
Evaluating the RHS of \eqref{eqn: gaps_lb} for each type of two-armed
bandit instance gives the required result. There are broadly two cases
involved here: $\nu$ being a feasible instance and $\nu$ being an
infeasible instance. The former has three subcases for the non-optimal
arm, i.e., arm 2: arm 2 is feasible suboptimal, deceiver, and
infeasible suboptimal. The general methodology used here is to
minimize both terms inside the maximum subject to the constraints on
the arms.

\noindent \textbf{Case 1: $\nu$ is feasible and $\nu(2)$ is feasible
  suboptimal}\\
We first evaluate the infimum over the two cases: $J(\nu^{\prime} )
\neq J(\nu), F(\nu^{\prime} ) = F(\nu)$, and $F(\nu^{\prime} ) \neq
F(\nu)$, and then find the minimum of these two cases. In the former
case, we have that $J(\nu^{\prime} ) \neq J(\nu), F(\nu^{\prime} ) =
F(\nu)$, i.e., both $\nu$ and $\nu^{\prime}$ are feasible but their
optimal arms are different, while in the latter case, we have that
$\nu^{\prime}$ is infeasible. We first consider the former case. WLOG,
we assume that $J(\nu)=1$ and $J(\nu^{\prime})=2$.
\begin{itemize}
\item \textit{Arm 1 of $\nu^{\prime}$ is feasible.}\\ In this case, we
  have that $\mu^{\prime}_2(1)\leq \tau$, and hence there are no
  restrictions on $\mu^{\prime}_2(1)$ and $\mu^{\prime}_2(2)$ (as long
  as they are below $\tau$). We thus set $\mu^{\prime}_2(1) = \mu_2(1)
  $ and $\mu^{\prime}_2(2)=\mu_2(2)$. It follows that:
\begin{align*}
&\underset{ \substack{\nu^{\prime} \in \mathcal{M}: \\ J(\nu^{\prime} ) =2  \\  \mu_2^{\prime}(2),\mu_2^{\prime}(1) < \tau  } }{\inf} \max \Bigl \{  a_1\left(\mu_{1}(1)-\mu_{1}^{\prime}(1)\right)^2 \\ &+  a_2\left(\mu_{2}(1)-\mu_{2}^{\prime}(1)\right)^2  ,  
a_1\left(\mu_{1}(2)-\mu_{1}^{\prime}(2)\right)^2 \\
&+  a_2\left(\mu_{2}(2)-\mu_{2}^{\prime}(2)\right)^2  \Bigr \} \\
&= \underset{ \substack{\nu^{\prime} \in \mathcal{M}: \\  \mu_1^{\prime}(2) \leq \mu_1^{\prime}(1)  } }{\inf}   \max \Bigl \{  a_1\left(\mu_{1}(1)-\mu_{1}^{\prime}(1)\right)^2   ,  \\
&a_1\left(\mu_{1}(2)-\mu_{1}^{\prime}(2)\right)^2  \Bigr \} \\
&= \frac{a_1(\mu_{1}(2) - \mu_{1}(1))^2}{4},
\end{align*}
where the infimum is attained midway between $\mu_{1}(1)$ and
$\mu_{1}(2)$.
\item \textit{Arm 1 of $\nu^{\prime}$ is infeasible.}\\ In this case,
  we have that $\mu^{\prime}_2(1)>\tau$, and hence there are no
  restrictions on $\mu^{\prime}_1(1)$, $\mu^{\prime}_1(2)$, and
  $\mu^{\prime}_2(2)$ (as long as it is below $\tau$). We thus set
  $\mu^{\prime}_1(1)=\mu_1(1)$, $\mu^{\prime}_1(2)=\mu_1(2)$ and
  $\mu^{\prime}_2(2)=\mu_2(2)$. Thus,
\begin{align*}
&\underset{ \substack{\nu^{\prime} \in \mathcal{M}: \\ J(\nu^{\prime} ) =2  \\  \mu_2^{\prime}(2) < \tau  < \mu_2^{\prime}(1)  } }{\inf} \max \Bigl \{  a_1\left(\mu_{1}(1)-\mu_{1}^{\prime}(1)\right)^2 \\ &+  a_2\left(\mu_{2}(1)-\mu_{2}^{\prime}(1)\right)^2  ,  
a_1\left(\mu_{1}(2)-\mu_{1}^{\prime}(2)\right)^2 \\
&+  a_2\left(\mu_{2}(2)-\mu_{2}^{\prime}(2)\right)^2  \Bigr \} \\
&= \underset{ \substack{\nu^{\prime} \in \mathcal{M}: \\ J(\nu^{\prime} ) =2  \\  \tau  < \mu_2^{\prime}(1)  } }{\inf}   a_2\left(\mu_{2}(1)-\mu_{2}^{\prime}(1)\right)^2    \\
&= a_2 (\tau - \mu_2(1))^2.
\end{align*}
\end{itemize}
It is enough to evaluate the infimum only for the case where
$J(\nu^{\prime} ) \neq J(\nu), F(\nu^{\prime} ) = F(\nu)$ because in
the case where $F(\nu^{\prime} ) \neq F(\nu)$, the infimum is at least
$a_2 \max \left \{ (\tau - \mu_2(1))^2, (\tau - \mu_2(2))^2 \right
\}$. Thus, combining the results of the cases discussed above, in the
case of a feasible instance with optimal arm being arm 1 and arm 2
being a suboptimal feasible arm, we have that
\begin{align*}
&\limsup_{T \rightarrow \infty}-\frac{1}{T} \log e_{T}(\nu) \leq \\
&\frac{1}{2} \min \left  \{a_2 (\tau - \mu_2(1))^2 , \frac{a_1(\mu_{1}(2) - \mu_{1}(1))^2}{4} \right \}.
\end{align*}

\noindent \textbf{Case 2: $\nu$ is feasible and $\nu(2)$ is a
  deceiver} \\
We first evaluate the infimum over the two cases: $J(\nu^{\prime} )
\neq J(\nu), F(\nu^{\prime} ) = F(\nu)$, and $F(\nu^{\prime} ) \neq
F(\nu)$; and then find the minimum of these two cases. In the former
case, we have that $J(\nu^{\prime} ) \neq J(\nu), F(\nu^{\prime} ) =
F(\nu)$, i.e., both $\nu$ and $\nu^{\prime}$ are feasible but their
optimal arms are different, while in the latter case, we have that
$\nu^{\prime}$ is infeasible. We first consider the former case. WLOG,
we assume that $J(\nu)=1$ and $J(\nu^{\prime})=2$.
\begin{itemize}
\item \textit{Arm 1 of $\nu^{\prime}$ is feasible.}\\ In this case, we
  have that $\mu^{\prime}_2(1)\leq \tau$. As we also have that
  $\mu_1(2) \leq \mu_1(1)$ and the only constraint on the first
  dimensions of the arms of instance $\nu^{\prime}$ is
  $\mu^{\prime}_1(2) \leq \mu^{\prime}_1(1)$, we set $\mu_1(2) =
  \mu^{\prime}_1(2)$ and $\mu_1(1) = \mu^{\prime}_1(1)$ to minimize
  each term inside the maximum. Thus,
\begin{align*}
&\underset{ \substack{\nu^{\prime} \in \mathcal{M}: \\ J(\nu^{\prime} ) =2  \\   \mu_2^{\prime}(1) ,\mu_2^{\prime}(2) < \tau  } }{\inf} \max \Bigl \{  a_1\left(\mu_{1}(1)-\mu_{1}^{\prime}(1)\right)^2 \\ &+  a_2\left(\mu_{2}(1)-\mu_{2}^{\prime}(1)\right)^2  ,  
a_1\left(\mu_{1}(2)-\mu_{1}^{\prime}(2)\right)^2 \\
&+  a_2\left(\mu_{2}(2)-\mu_{2}^{\prime}(2)\right)^2  \Bigr \} \\
&= \underset{ \substack{\nu^{\prime} \in \mathcal{M}: \\ J(\nu^{\prime} ) =2  \\   \mu_2^{\prime}(1) ,\mu_2^{\prime}(2) < \tau   } }{\inf}   \max \Bigl \{  a_2\left(\mu_{2}(1)-\mu_{2}^{\prime}(1)\right)^2  ,  \\
&  a_2\left(\mu_{2}(2)-\mu_{2}^{\prime}(2)\right)^2  \Bigr \} \\
&= \underset{ \substack{\nu^{\prime} \in \mathcal{M}: \\ J(\nu^{\prime} ) =2  \\   \mu_2^{\prime}(1) ,\mu_2^{\prime}(2) < \tau   } }{\inf}    a_2\left(\mu_{2}(2)-\mu_{2}^{\prime}(2)\right)^2    \\
&= a_2 (\mu_2(2) - \tau  )^2 .
\end{align*}

\item \textit{Arm 1 of $\nu^{\prime}$ is infeasible.}\\ In this case,
  we have that $\mu^{\prime}_2(1)>\tau$, and hence there are no
  restrictions on $\mu^{\prime}_1(1)$ and $\mu^{\prime}_1(2)$. In this
  case, we set $\mu^{\prime}_1(1)=\mu_1(1)$,
  $\mu^{\prime}_1(2)=\mu_1(2)$. Thus,
\begin{align*}
&\underset{ \substack{\nu^{\prime} \in \mathcal{M}: \\ J(\nu^{\prime} ) =2  \\  \mu_2^{\prime}(2) < \tau  < \mu_2^{\prime}(1)  } }{\inf} \max \Bigl \{  a_1\left(\mu_{1}(1)-\mu_{1}^{\prime}(1)\right)^2 \\ &+  a_2\left(\mu_{2}(1)-\mu_{2}^{\prime}(1)\right)^2  ,  
a_1\left(\mu_{1}(2)-\mu_{1}^{\prime}(2)\right)^2 \\
&+  a_2\left(\mu_{2}(2)-\mu_{2}^{\prime}(2)\right)^2  \Bigr \} \\
&= \underset{ \substack{\nu^{\prime} \in \mathcal{M}: \\ J(\nu^{\prime} ) =2  \\  \mu_2^{\prime}(2) < \tau  < \mu_2^{\prime}(1)  } }{\inf}   \max \Bigl \{  a_2\left(\mu_{2}(1)-\mu_{2}^{\prime}(1)\right)^2  ,  \\
&  a_2\left(\mu_{2}(2)-\mu_{2}^{\prime}(2)\right)^2  \Bigr \} \\
&= a_2 \max \left \{  (\tau - \mu_2(1))^2, (\mu_2(2) - \tau  )^2 \right \}.
\end{align*}
\end{itemize}

We now consider the case where $F(\nu^{\prime} ) \neq F(\nu)$, i.e.,
$\nu^{\prime}$ is an infeasible instance. Here, as there are no
constraints on arm 2 of the instance $\nu^{\prime}$ apart from
$\mu^{\prime}_2(2)\geq \mu^{\prime}_2(1)\geq \tau$, we set
$\mu^{\prime}(2)=\mu(2)$. We also set $\mu^{\prime}_1(1)=\mu_1(1)$ as
the only constraints on arm 1 of the instance $\nu^{\prime}$ is that
$\mu^{\prime}_2(2)\geq \mu^{\prime}_2(1)\geq \tau$. Thus,
\begin{align*}
&\underset{ \substack{\nu^{\prime} \in \mathcal{M}: \\    \mu_2^{\prime}(1) ,\mu_2^{\prime}(2) > \tau  } }{\inf} \max \Bigl \{  a_1\left(\mu_{1}(1)-\mu_{1}^{\prime}(1)\right)^2 \\ &+  a_2\left(\mu_{2}(1)-\mu_{2}^{\prime}(1)\right)^2  ,  
a_1\left(\mu_{1}(2)-\mu_{1}^{\prime}(2)\right)^2 \\
&+  a_2\left(\mu_{2}(2)-\mu_{2}^{\prime}(2)\right)^2  \Bigr \} \\
&=\underset{ \substack{\nu^{\prime} \in \mathcal{M}: \\ J(\nu^{\prime} ) =2  \\   \mu_2^{\prime}(1) ,\mu_2^{\prime}(2) > \tau  } }{\inf}  a_2\left(\mu_{2}(1)-\mu_{2}^{\prime}(1)\right)^2   \\
&= a_2 (\tau - \mu_2(1) )^2 .
\end{align*}
Thus, combining the results of the three cases above, we have that for
a two-armed feasible instance $\nu$ with arm 1 being the optimal arm
and arm 2 being a deceiver arm,
\begin{align*}
&\limsup_{T \rightarrow \infty}-\frac{1}{T} \log e_{T}(\nu) \leq \\
&\frac{1}{2} \min \left  \{a_2 (\tau - \mu_2(1))^2 , a_2 (\mu_2(2) - \tau  )^2 \right \}.
\end{align*}

\noindent \textbf{Case 3: $\nu$ is feasible and $\nu(2)$ is infeasible
  suboptimal} \\ We first evaluate the infimum over the two cases:
$J(\nu^{\prime} ) \neq J(\nu), F(\nu^{\prime} ) = F(\nu)$, and
$F(\nu^{\prime} ) \neq F(\nu)$; and then find the minimum over these
two cases. In the former case, we have that $J(\nu^{\prime} ) \neq
J(\nu), F(\nu^{\prime} ) = F(\nu)$, i.e., both $\nu$ and
$\nu^{\prime}$ are feasible but their optimal arms are different,
while in the latter case, we have that $\nu^{\prime}$ is
infeasible. We first consider the former case. WLOG, we assume that
$J(\nu)=1$ and $J(\nu^{\prime})=2$.
\begin{enumerate}
\item \textit{Arm 1 of $\nu^{\prime}$ is feasible.} \\ In this case,
  as there are no restrictions on the second dimensions of the arms of
  $\nu^{\prime}$ apart from them being smaller than $\tau$, we set
  $\mu^{\prime}_2(1)=\mu_2(1)$, $\mu^{\prime}_2(2)=\tau$. Also, as
  $\mu^{\prime}_1(2) \leq \mu^{\prime}_1(1)$, to attain the infimum in
  \eqref{eqn: gaps_lb}, it is clear that $\mu^{\prime}_1(2) \geq
  \mu_1(1)$ and $\mu^{\prime}_1(1) \leq \mu_1(2)$. Hence, we also set
  $\mu^{\prime}_1(2) = \mu^{\prime}_1(1)$. Thus,
\begin{align*}
&\underset{ \substack{\nu^{\prime} \in \mathcal{M}: \\J(\nu^{\prime} ) =2  \\  \mu_2^{\prime}(1) ,\mu_2^{\prime}(2) < \tau   } }{\inf} \max \Bigl \{  a_1\left(\mu_{1}(1)-\mu_{1}^{\prime}(1)\right)^2 \\ &+  a_2\left(\mu_{2}(1)-\mu_{2}^{\prime}(1)\right)^2  ,  
a_1\left(\mu_{1}(2)-\mu_{1}^{\prime}(2)\right)^2 \\
&+  a_2\left(\mu_{2}(2)-\mu_{2}^{\prime}(2)\right)^2  \Bigr \} \\
&=\underset{ \substack{\nu^{\prime} \in \mathcal{M}: \\ J(\nu^{\prime} ) =2  \\  \mu_2^{\prime}(1) ,\mu_2^{\prime}(2) < \tau  } }{\inf} \max \Bigl \{  a_1\left(\mu_{1}(1)-\mu_{1}^{\prime}(1)\right)^2, \\ &  
a_1\left(\mu_{1}(2)-\mu_{1}^{\prime}(2)\right)^2 \\
&+  a_2\left(\mu_{2}(2)-\tau \right)^2  \Bigr \}. \\
\end{align*}
For notational simplicity, let $M$ denote
$\sqrt{a_1}(\mu_1(2)-\mu_1(1))$, $y$ denote
$\sqrt{a_2}(\mu_2(2)-\tau)$ and $x$ denote
$\sqrt{a_1}(\mu^{\prime}_1(1)-\mu_1(1))$. Note that $x<M$. Then the
expression to be evaluated is:
\begin{align*}
\underset{x}{\inf} \max \{  x^2,y^2 +(M-x)^2 \}.
\end{align*}
We consider the following two cases: 
\begin{enumerate}
\item $y>M$. \\
We have that 
\begin{align*}
\frac{y^2+M^2}{2M} >M. 
\end{align*}
As we also have that $x<M$, 
\begin{align*}
x<\frac{y^2+M^2}{2M},
\end{align*}
which gives that 
\begin{align*}
y^2+(M-x)^2 > x^2.
\end{align*}
Thus,
\begin{align*}
\underset{x}{\inf}  \max \{  x^2,y^2+(M-x)^2 \} &=  \underset{x}{\inf} \{ y^2+(M-x)^2 \} \\
&= y^2.
\end{align*}
\item $y \leq M$. \\
We have that
\begin{align*}
\frac{y^2+M^2}{2M}\leq M.
\end{align*}   
If
\begin{align*}
    x \leq \frac{y^2+M^2}{2M}, 
\end{align*}  
then
\begin{align*}
 \max \{  x^2,y^2 +(M-x)^2 \} &=y^2 +(M-x)^2, 
 \end{align*} 
which is minimum at  $x=\frac{y^2+M^2}{2M}$. Also, if
\begin{align*}
    x &\geq \frac{y^2+M^2}{2M},
 \end{align*} 
 then
\begin{align*}
\max \{  x^2,y^2 +(M-x)^2 \} &= x^2, 
 \end{align*} 
 which is minimum at $x=\frac{y^2+M^2}{2M}$. In both cases, 
 \begin{align*}
\underset{x}{\inf} \max \{  x^2,y^2+(M-x)^2 \} &= \left ( \frac{y^2+M^2}{2M}  \right)^2.
\end{align*}
\end{enumerate}
Combining the two cases $y>M$ and $y\leq M$, we have that
\begin{align*}
\underset{x}{\inf} \max \{  x^2,y^2 +(M-x)^2 \}= \left ( \frac{y^2+z^2}{2z}  \right)^2 ,
\end{align*}
 where $z=\max \{ y, M \}$. We also have that
 \begin{align*}
\frac{z^2}{4} \leq \left ( \frac{y^2+z^2}{2z}  \right)^2 \leq z^2, 
\end{align*}
i.e., the infimum is within a constant factor of $z^2$. As our
algorithm is motivated by these gaps and to avoid comparisons between
the first and the second dimensions, we have that
 \begin{align*}
\underset{x}{\inf} \max \{  x^2,y^2 +(M-x)^2 \} \leq z^2.
\end{align*}
\item \textit{Arm 1 of $\nu^{\prime}$ is infeasible.} \\ In this case,
  as there are no restrictions on the first dimensions of the arms of
  $\nu^{\prime}$, we set $\mu^{\prime}_1(1)=\mu_1(1)$,
  $\mu^{\prime}_1(2)=\mu_1(2)$. Thus,
\begin{align*}
&\underset{ \substack{\nu^{\prime} \in \mathcal{M}: \\ J(\nu^{\prime} ) =2  \\    \mu_2^{\prime}(2) < \tau  < \mu_2^{\prime}(1) } }{\inf} \max \Bigl \{  a_1\left(\mu_{1}(1)-\mu_{1}^{\prime}(1)\right)^2 \\ &+  a_2\left(\mu_{2}(1)-\mu_{2}^{\prime}(1)\right)^2,  
a_1\left(\mu_{1}(2)-\mu_{1}^{\prime}(2)\right)^2 \\
&+  a_2\left(\mu_{2}(2)-\mu_{2}^{\prime}(2)\right)^2  \Bigr \} \\
&=\underset{ \substack{\nu^{\prime} \in \mathcal{M}: \\ J(\nu^{\prime} ) =2  \\    \mu_2^{\prime}(2) < \tau  < \mu_2^{\prime}(1)  } }{\inf} \max \Bigl \{   a_2\left(\mu_{2}(1)-\mu_{2}^{\prime}(1)\right)^2,  \\
& a_2\left(\mu_{2}(2)-\mu_{2}^{\prime}(2)\right)^2  \Bigr \} \\
&= a_2 \max \left \{  (\tau - \mu_2(1))^2, (\mu_2(2) - \tau  )^2 \right \}.
\end{align*}
\end{enumerate}
Next, we consider the case where $F(\nu^{\prime} ) \neq F(\nu)$, i.e.,
$\nu^{\prime}$ is an infeasible instance.  In this case, as there are
no restrictions on the first dimensions of the arms of $\nu^{\prime}$,
we set $\mu^{\prime}_1(1)=\mu_1(1)$,
$\mu^{\prime}_1(2)=\mu_1(2)$. Moreover, as $\mu_2(2)>\tau$ and
$\mu_2^{\prime}(2)>\tau$, we set $\mu_2(2)= \mu_2^{\prime}(2)$. Thus,
\begin{align*}
&\underset{ \substack{\nu^{\prime} \in \mathcal{M}: \\    \mu_2^{\prime}(1), \mu_2^{\prime}(2) > \tau  } }{\inf} \max \Bigl \{  a_1\left(\mu_{1}(1)-\mu_{1}^{\prime}(1)\right)^2 \\ &+  a_2\left(\mu_{2}(1)-\mu_{2}^{\prime}(1)\right)^2,  
a_1\left(\mu_{1}(2)-\mu_{1}^{\prime}(2)\right)^2 \\
&+  a_2\left(\mu_{2}(2)-\mu_{2}^{\prime}(2)\right)^2  \Bigr \} \\
&=\underset{ \substack{\nu^{\prime} \in \mathcal{M}: \\ \mu_2^{\prime}(1), \mu_2^{\prime}(2) > \tau  } }{\inf}    a_2\left(\mu_{2}(1)-\mu_{2}^{\prime}(1)\right)^2  \\
&= a_2 (\tau - \mu_2(1))^2 .
\end{align*}

Thus, combining the results of all the cases discussed above, we have
that for a two-armed feasible instance $\nu$ with arm 1 being the
optimal arm and arm 2 being infeasible suboptimal,
\begin{align*}
&\limsup_{T \rightarrow \infty}-\frac{1}{T} \log e_{T}(\nu) \leq \\
& \quad \min \Bigl  \{a_2 (\tau - \mu_2(1))^2 , \\ &\max  \left \{ a_2(\mu_2(2)-\tau)^2, a_1(\mu_1(2)-\mu_1(1))^2 \right \}  \Bigr \}.
\end{align*} 

\noindent \textbf{Case 4: $\nu$ is infeasible} \\ We first evaluate
the infimum over the two cases: $J(\nu^{\prime} ) \neq J(\nu),
F(\nu^{\prime} ) = F(\nu)$, and $F(\nu^{\prime} ) \neq F(\nu)$; and
then find the minimum of these two cases. In the former case, we have
that $J(\nu^{\prime} ) \neq J(\nu), F(\nu^{\prime} ) = F(\nu)$, i.e.,
both $\nu$ and $\nu^{\prime}$ are infeasible but their optimal arms
are different, while in the latter case, we have that $\nu^{\prime}$
is feasible. We first consider the former case. WLOG, we assume that
$J(\nu)=1$ and $J(\nu^{\prime})=2$.

As there are no restrictions on $\mu^{\prime}_1(1)$ and
$\mu^{\prime}_1(2)$, we set $\mu^{\prime}_1(1)=\mu_1(1)$ and
$\mu^{\prime}_1(2)=\mu_1(2)$. It follows that:
\begin{align*}
&\underset{ \substack{\nu^{\prime} \in \mathcal{M}: \\ J(\nu^{\prime} ) =2  \\  \mu_2^{\prime}(2),\mu_2^{\prime}(1) > \tau } }{\inf} \max \Bigl \{  a_1\left(\mu_{1}(1)-\mu_{1}^{\prime}(1)\right)^2 \\ &+  a_2\left(\mu_{2}(1)-\mu_{2}^{\prime}(1)\right)^2  ,  
a_1\left(\mu_{1}(2)-\mu_{1}^{\prime}(2)\right)^2 \\
&+  a_2\left(\mu_{2}(2)-\mu_{2}^{\prime}(2)\right)^2  \Bigr \} \\
&= \underset{ \substack{\nu^{\prime} \in \mathcal{M}: \\  \tau < \mu_2^{\prime}(1) \leq \mu_2^{\prime}(2)   } }{\inf}   \max \Bigl \{  a_2\left(\mu_{2}(1)-\mu_{2}^{\prime}(1)\right)^2   ,  \\
&a_2 \left(\mu_{2}(2)-\mu_{2}^{\prime}(2)\right)^2  \Bigr \} \\
&= \frac{a_1(\mu_{2}(2) - \mu_{2}(1))^2}{4},
\end{align*}
where the infimum is attained midway between $\mu_{2}(1)$ and $\mu_{2}(2)$.

Next, we consider the case where $F(\nu^{\prime} ) \neq F(\nu)$, i.e.,
$\nu^{\prime}$ is a feasible instance. As the only restriction is that
at least one arm of $\nu^{\prime}$ is feasible, we set
$\mu^{\prime}_2=\mu_2$ and make arm 1 feasible. Thus, we have that
\begin{align*}
&\underset{ \substack{\nu^{\prime} \in \mathcal{M}: \\ \mathcal{K}(\nu^{\prime}) = \emptyset } }{\inf} \max \Bigl \{  a_1\left(\mu_{1}(1)-\mu_{1}^{\prime}(1)\right)^2 \\ &+  a_2\left(\mu_{2}(1)-\mu_{2}^{\prime}(1)\right)^2  ,  
a_1\left(\mu_{1}(2)-\mu_{1}^{\prime}(2)\right)^2 \\
&+  a_2\left(\mu_{2}(2)-\mu_{2}^{\prime}(2)\right)^2  \Bigr \} \\
&\leq \underset{ \substack{\nu^{\prime} \in \mathcal{M}: \\  \tau > \mu_2^{\prime}(1)    } }{\inf}    a_1\left(\mu_{1}(1)-\mu_{1}^{\prime}(1)\right)^2 +  a_2\left(\mu_{2}(1)-\mu_{2}^{\prime}(1)\right)^2 \\
&= a_2 (\mu_2(1) - \tau)^2.
\end{align*}
Thus, combining the results of the two cases discussed above, we have
that for a two-armed infeasible instance $\nu$ with optimal arm being
arm 1,
\begin{align*}
&\limsup_{T \rightarrow \infty}-\frac{1}{T} \log e_{T}(\nu) \leq \\
& \min \Bigl  \{\frac{a_1(\mu_{2}(2) - \mu_{2}(1))^2}{4} , a_2 (\mu_2(1) - \tau)^2 \Bigr \}.
\end{align*} 

\end{proof}

%% file: sections/appendix_C.tex
\section{Complete proof of Theorem \ref{thm: ub}}
\label{app: ub}

In this section, we complete the proof for Theorem~\ref{thm: ub} by
proving Lemmas~\ref{lemma:2arm-upper_bound}
and~\ref{lemma:3arm-upper_bound} for feasible and infeasible instances
separately.

\subsection{Underlying instance is feasible}

%%%%%%%%%%%%%%%%%%%%%%%%%%%%%%%%%%%%%%%%
%%%%%%%%%%%%%%%%%%%%%%%%%%%%%%%%%%%%%%%%
%%%%%%%%%%%%%%%%%%%%%%%%%%%%%%%%%%%%%%%%

\begin{proof}[Proof of Lemma \ref{lemma:2arm-upper_bound}]
Here, the feasible instance consists of two arms and each arm has been
drawn $n_1$ times. Let $\mathcal{B}_k$ denote the event that arm 1 is
empirically feasible at the end of round $k$. Thus,
\begin{align*}
\prob{\mathcal{A}_1} &= \prob{\mathcal{A}_1 \cap \mathcal{B}_1} + \prob{\mathcal{A}_1 \cap \mathcal{B}_1^{\mathsf{c}}} \\
%&\leq \prob{\mathcal{A}_1 \cap \mathcal{B}_1} + \prob{ \mathcal{B}_1^{\mathsf{c}}} \\
&\leq \prob{\mathcal{A}_1 \cap \mathcal{B}_1} + \prob{ \mu_2^{1}(1) >
  \tau }. \shortintertext{Using \eqref{eqn: conc_inequality} to bound
  the last term from above, we get:} \prob{\mathcal{A}_1} &\leq
\prob{\mathcal{A}_1 \cap \mathcal{B}_1} + 2 \exp \left ( -a_2 n_{1}
(\tau-\mu_2(1))^2\right ) \\ &\leq \prob{\mathcal{A}_1 \cap
  \mathcal{B}_1} + 2 \exp \left ( -n_{1} \Delta(1,2)^2 \right
). \numberthis \label{eqn: 2 arm ub main}
\end{align*}
 The event $\{{\mathcal{A}_1 \cap \mathcal{B}_1}\}$ corresponds to the
 set of outcomes where arm 1 is empirically feasible at the end of
 round 1 and is still rejected. This would require that arm 2 be
 empirically feasible and also be the empirically optimal arm. Thus,
 we get
 \begin{equation}
\prob{\mathcal{A}_1 \cap \mathcal{B}_1} \leq \prob{\hat{\mu}_2^1(2) \leq \tau, \hat{\mu}_1^1(1) > \hat{\mu}_1^1(2)}.  \label{eqn: consr_2arm_1_feasible}
 \end{equation}
 This can be bounded using \eqref{eqn: conc_inequality} depending upon
 the nature of arm 2.

\noindent {\bf Case 1: Arm 2 is a feasible suboptimal arm.} Using
\eqref{eqn: consr_2arm_1_feasible},
 \begin{align*}
\prob{\mathcal{A}_1 \cap \mathcal{B}_1} &\leq \prob{\hat{\mu}_1^1(1) > \hat{\mu}_1^1(2)} \\
&= \mathbb{P} \Bigl ( (\hat{\mu}_1^1(1) - \mu_1(1))- (\hat{\mu}_1^1(2) - \mu_1(2)) \\
&\qquad\qquad > (\mu_1(2)-\mu_1(1)) \Bigr ) \\
&\leq \prob{ (\hat{\mu}_1^1(1) - \mu_1(1))> \frac{(\mu_1(2)-\mu_1(1))}{2} } \\
&+ \prob{ (\hat{\mu}_1^1(2) - \mu_1(2)) < - \frac{(\mu_1(2)-\mu_1(1))}{2} },
\end{align*}
where the last step follows from the fact that both $(\hat{\mu}_1^1(1)
- \mu_1(1))$ and $( \mu_1(2) - \hat{\mu}_1^1(2)) $ cannot be greater
than $\frac{(\mu_1(2)-\mu_1(1))}{2} $ and a subsequent union bounding
argument. Thus, using \eqref{eqn: conc_inequality}:
\begin{align*}
\prob{\mathcal{A}_1 \cap \mathcal{B}_1} &\leq 4 \exp(-\frac{n_1}{4} a_1(\mu_1(2)-\mu_1(1))^2)  \\
&\leq 4 \exp(- \frac{n_1}{4} \Delta(1,2)^2). \numberthis \label{eqn: 3 arm ub feasible subopt}
\end{align*}

 \noindent {\bf Case 2: Arm 2 is a deceiver arm.}
 Similarly, using \eqref{eqn: consr_2arm_1_feasible},
  \begin{align*}
\prob{\mathcal{A}_1 \cap \mathcal{B}_1} &\leq \prob{\hat{\mu}_2^1(2) \leq \tau} \\
&\leq 2 \exp \left ( -n_{1} a_2 (\mu_2(2)-\tau)^2 \right )  \\
&\leq 2 \exp \left ( -n_{1} \Delta(1,2)^2 \right ).
  \end{align*}

\noindent {\bf Case 3: Arm 2 is an infeasible suboptimal arm.} This
case follows from Case~1 if $\delta(1,2)$ is dictated by the the
suboptimality gap of arm~2, and from Case~2 if $\delta(1,2)$ is
dictated by the the infeasibility gap of arm~2.

\ignore{
From~\eqref{eqn: consr_2arm_1_feasible},
 \begin{align*}
& \prob{\mathcal{A}_1 \cap \mathcal{B}_1} \leq  \min \left \{ \prob{ \hat{\mu}_2^1(2) \leq \tau}, \prob{ \hat{\mu}_1^1(1) > \hat{\mu}_1^1(2)} \right \} \\
&\leq 4 \exp(- \frac{n_1}{4} \max \left \{ a_1(\mu_1(2)-\mu_1(1))^2), a_2 (\mu_2(2)-\tau)^2 \right \} ) \\
 &= 4 \exp(-\frac{n_1}{4} \delta(1,2)^2). 
 \end{align*}
} % end ignore

The statement of the lemma now follows, combining these three cases
with~\eqref{eqn: 2 arm ub main}.
\ignore{Substituting the above in \eqref{eqn: 2 arm ub main} gives us:
\begin{align*}
\prob{\mathcal{A}_1} &\leq 4 \exp(-\frac{n_1}{4} \delta(1,2)^2) + 2 \exp \left ( -n_{1} \Delta(1,2)^2\right ) \\
&\leq 6 \exp(-\frac{n_1}{4} \Delta(1,2)^2). 
\end{align*}
} %end ignore
\end{proof}

%%%%%%%%%%%%%%%%%%%%%%%%%%%%%%%%%%%%%%%%
%%%%%%%%%%%%%%%%%%%%%%%%%%%%%%%%%%%%%%%%
%%%%%%%%%%%%%%%%%%%%%%%%%%%%%%%%%%%%%%%%

\begin{proof}[Proof of Lemma~\ref{lemma:3arm-upper_bound}]
The feasible instance consists of three arms, each of which has been
drawn $n_1$ times. Let $\mathcal{B}_k$ denote the event that arm 1 is
empirically feasible at the end of round $k$. Proceeding similarly as
in \eqref{eqn: 2 arm ub main},
\begin{align*}
&\prob{ \{\hat{J}(A_k) = 2\} \cap \mathcal{A}_1} \\
  &\leq \prob{   \mathcal{A}_1 \cap \mathcal{B}_1 \cap \{\hat{J}(A_k) = 2\}} +  2 \exp \left ( - n_{1} \Delta(1,3)^2 \right ) \\
  &=: \prob{\G} + 2 \exp \left ( - n_{1} \Delta(1,3)^2 \right ). \numberthis \label{eqn: 3 arm ub main}
\end{align*}
The term $\prob{\G}$ can be bounded depending upon the nature of arm~3.

\noindent {\bf Case 1: Arm 3 is a feasible suboptimal arm.}

\noindent Note that $\prob{\G}$ is the probability of arm 1 being
rejected at the end of round 1, arm 2 being empirically optimal, and
arm 1 looking empirically feasible. Event~$\G$ thus implies that arm~2
is also empirically feasible, and moreover, has a lower value of the
objective attribute~$\hat\mu_1^1(\cdot)$ than arm~1. We now further
decompose $\prob{\G}$ as follows:
\begin{align*}
  \prob{\G} =
  &\prob{\G \cap\{\hat\Delta(2,1) = \hat\Delta(2,3)\}} \\
  &\quad + \prob{\G \cap\{\hat\Delta(2,1) > \hat\Delta(2,3)\}} =: \prob{\G_1} + \prob{\G_2}
\end{align*}

{\bf Bounding~$\prob{\G_1}$:} The event~$\G_1$ implies that
arm~1 is rejected based on our tie breaking rule. This can only occur
if arm~3 appears empirically feasible, and moreover, appears superior
to arm~1 on the objective attribute $\hat\mu_1^1(\cdot).$ Thus,
\begin{align*}
 \prob{\G_1} \leq \prob{\hat\mu_1^1(1) \geq  \hat\mu_1^1(3)} \leq 4 \exp(-\frac{n_1}{4} \delta(1,3)^2),
\end{align*}
where the last step follows from \eqref{eqn: 3 arm ub feasible subopt}.

{\bf Bounding~$\prob{\G_2}$:} The event~$\G_1$ implies that
arm~1 is rejected based on its estimated suboptimality gap alone. In
this case, we must have
%In this case, we have that $\hat\Delta(2,3) = \hat\delta(2,3)$. Note
%that $\hat\Delta(2,3) = \tau - \hat\mu^1_2(2)$ would imply that
%$\hat\Delta(2,1) = \tau - \hat\mu^1_2(2)$ as $\hat\Delta(2,1) \geq
%\hat\Delta(2,3)$ and $\hat\Delta(2,1) \leq \tau - \hat\mu^1_2(2)$,
%which has already been covered in the previous case. Thus, we have
%that
\begin{align*}
 \hat\delta(2,3) = \hat\Delta(2,3) < \hat\Delta(2,1) \leq \hat\delta(2,1) = \sqrt{a_1} \left(\hat\mu_1^1(1) - \hat\mu_1^1(2)\right).
\end{align*}
Thus, $\prob{\G_2}$ is upper bounded by
%\begin{align*}
%  &\prob{   \mathcal{A}_1 \cap \mathcal{B}_1 \cap \{\hat{J}(A_k) = 2\}} \\
% &= \prob{ \{\hat\Delta(2,1) \geq \hat\Delta(2,3) \} \cap \mathcal{B}_1 \cap \{\hat{J}(A_k) = 2\}} \\
%  &\leq  \prob{ \{\hat\delta(2,1) \geq \hat\delta(2,3) \} \cap \mathcal{B}_1 \cap \{\hat{J}(A_k) = 2\}}  \\
%  &=  \prob{ \{ \sqrt{a_1} ( \hat\mu_1^1(1) - \hat\mu_1^1(2) ) \geq \hat\delta(2,3) \} \cap \mathcal{B}_1 \cap \{\hat{J}(%A_k) = 2\}}. 
%\end{align*}
\begin{equation*}
  \prob{\{\sqrt{a_1} \left(\hat\mu_1^1(1) - \hat\mu_1^1(2)\right) > \hat\delta(2,3) \} \cap \mathcal{B}_1 \cap \{\hat{J}(A_k) = 2\}}.
\end{equation*}
Note that for $\sqrt{a_1} ( \hat\mu_1^1(1) - \hat\mu_1^1(2) ) \geq
\hat\delta(2,3) $ to happen when arm~2 is empirically optimal and
arm~1 is empirically feasible, it cannot be that arm~3 has a higher
$\hat{\mu}^{1}_1(\cdot)$ than arm~1 (i.e., arm~3 appears inferior on
the objective attribute), regardless of whether arm 3 is empirically
feasible or infeasible.
%This is because, when arm 2 is empirically
%optimal and arm 1 is empirically feasible, if arm 3 had a higher
%$\hat{\mu}^{1}_1(\cdot)$ than arm 1, then
%\begin{align*}
%\hat\delta(2,3) &\geq \sqrt{a_1} \left ( \hat{\mu}^1_1(3)-\hat{\mu}^1_1(2) \right ) \\
%&\geq \sqrt{a_1} \left ( \hat{\mu}^1_1(1)-\hat{\mu}^1_1(2) \right ).
%\end{align*}
Thus, we have that:
\begin{align*}
\prob{\G_2} \leq \prob{\hat{\mu}_1^1(3) \leq \hat{\mu}_1^1(1) } \leq 4 \exp \left(-\frac{n_1}{4} \delta(1,3)^2 \right).
\end{align*}

Thus, for a feasible instance where arm 3 is suboptimal, the
probability of arm 1 being rejected at the end of the first round when
arm 2 is empirically optimal can be bounded as follows, combining
\eqref{eqn: 3 arm ub main} with our bounds on $\prob{\G_1}$ and
$\prob{\G_2}:$
\begin{equation}
  \prob{ \{\hat{J}(A_k) = 2\} \cap \mathcal{A}_1} \leq 10 \exp \left ( - \frac{n_1}{4} \Delta(1,3)^2 \right ).
  \label{eqn: 3 arm ub feasible subopt main}
\end{equation}

\noindent {\bf Case 2: Arm 3 is a deceiver arm.}

Next, we consider the case where arm 3 is an infeasible arm. We
bound~$\prob{\G}$ in the following way:
\begin{align}
  \prob{\G} &= \prob{\G \cap \{\hat{\mu}^1_1(3) \leq \tau \}} + \prob{\G \cap \{\hat{\mu}^1_1(3) > \tau \}} \nonumber \\
  &\leq \prob{\hat{\mu}^1_1(3) \leq \tau} + \prob{\G \cap \{\hat{\mu}^1_1(3) > \tau \}} \nonumber \\
  &\leq 2 \exp (-n_1 \Delta(1,3)^2) + \prob{\G \cap \{\hat{\mu}^1_1(3) > \tau \}} \nonumber \\
  &=2 \exp (-n_1 \Delta(1,3)^2) \nonumber \\
  &\quad + \prob{\G \cap \{\hat{\mu}^1_1(3) > \tau \} \cap \{\hat\Delta(2,1) = \hat\Delta(2,3)\}} \nonumber \\
  &\quad + \prob{\G \cap \{\hat{\mu}^1_1(3) > \tau \} \cap \{\hat\Delta(2,1) > \hat\Delta(2,3)\}} \nonumber \\
  &=: 2 \exp (-n_1 \Delta(1,3)^2) + \prob{\G_1} + \prob{\G_2} 
\label{eq:lemma4_2}
\end{align}

\ignore{
\begin{align*}
&\prob{   \mathcal{A}_1 \cap \mathcal{B}_1 \cap \{\hat{J}(A_k) = 2\}} \\
&=\prob{   \mathcal{A}_1 \cap \{\hat{\mu}^1_1(3) \leq \tau \} \cap\mathcal{B}_1 \cap \{\hat{J}(A_k) = 2\}}  \\
&+ \prob{   \mathcal{A}_1 \cap  \{\hat{\mu}^1_1(3) > \tau \} \cap \mathcal{B}_1 \cap \{\hat{J}(A_k) = 2\}} \\
&\leq \prob{\{\hat{\mu}^1_1(3) \leq \tau \} } \\
&+\prob{   \mathcal{A}_1 \cap  \{\hat{\mu}^1_1(3) > \tau \} \cap \mathcal{B}_1 \cap \{\hat{J}(A_k) = 2\}} \\
%&+ \prob{ \{ \hat\delta(2,3) < \hat\delta(2,1) \} \cap  \{\hat{\mu}^1_1(3) > \tau \}  \, | \, \mathcal{B}_1 \cap \{\hat{J}(A_k) = 2\} }\\
&\leq 2 \exp (-n_1 a_2 (\mu_2(3)-\tau)^2 ) \\
&+ \prob{   \mathcal{A}_1 \cap \{\hat{\mu}^1_1(3) > \tau \} \cap \mathcal{B}_1 \cap \{\hat{J}(A_k) = 2\}} \\
&\leq 2 \exp (-n_1 \Delta(1,3)^2 ) \\
&+ \prob{   \mathcal{A}_1 \cap \{\hat{\mu}^1_1(3) > \tau \} \cap \mathcal{B}_1 \cap \{\hat{J}(A_k) = 2\}},
\end{align*}
}%end ignore
%where the second last step follows from \eqref{eqn:
%conc_inequality}. The second term in the above expression is the
%probability of arm 1 being rejected at the end of round 1, arm 3
%being empirically infeasible, arm 1 being empirically feasible and
%arm 2 being the empirically optimal arm. We again consider the
%following two cases while bounding this term:

Note that $\prob{\G_1} = 0,$ since when there is a tie in the
estimated suboptimality gaps of arms~1 and~3, with arm~1 appearing
feasible, and arm~3 appearing infeasible, arm~3 would get
rejected. Next, we bound~$\prob{\G_2}.$

\ignore{In this case, empirically infeasible arms would be rejected
  first in decreasing order of their $\hat\mu_2^1(\cdot)$ and
  empirically feasible arms would be rejected second in decreasing
  order of their $\hat\mu_1^1(\cdot)$. Thus, for arm 1 to be rejected
  when arm 2 is empirically optimal and arm 1 is empirically feasible,
  arm 3 has to be empirically feasible and hence, when we also have
  that arm 3 is empirically infeasible, this case does not contribute
  to arm 1 being rejected at the end of round 1.} 

Under the event~$\G_2,$ we have
\begin{equation}
\hat\Delta(2,1) > \hat\Delta(2,3) =
\hat\delta(2,3) \geq \sqrt{a_2} \left ( \hat{\mu}_2^1(3) - \tau \right
).
\label{eq:lemma4_1}
\end{equation}
\ignore{ Similar to the case of arm 3 being a suboptimal arm, we have
  that $\hat\Delta(2,3) = \hat\delta(2,3)$. We also have that
\begin{align*}
&\prob{   \mathcal{A}_1 \cap \{\hat{\mu}^1_1(3) > \tau \} \cap \mathcal{B}_1 \cap \{\hat{J}(A_k) = 2\}} \\
&= \prob{   \{ \hat\Delta(2,1) \geq \hat\Delta(2,3) \} \cap \{\hat{\mu}^1_1(3) > \tau \} \cap \mathcal{B}_1 \cap \{\hat{J}(A_k) = 2\}} \\
&= \prob{   \{ \hat\Delta(2,1) \geq \hat\delta(2,3) \} \cap \{\hat{\mu}^1_1(3) > \tau \} \cap \mathcal{B}_1 \cap \{\hat{J}(A_k) = 2\}},
\end{align*}
where 
\begin{align*}
\hat\Delta(2,1) &\leq  \sqrt{a_2} ( \tau - \hat\mu_2^1(2) ), \\
\hat\delta(2,1) &= \sqrt{a_1} \left ( \hat{\mu}_1^1(1) - \hat{\mu}_1^1(2) \right ), \\
\hat\delta(2,3) &\geq  \sqrt{a_2} \left ( \hat{\mu}_2^1(3) - \tau \right ).  
\end{align*}
} %end ignore
We now take two cases for the nature of arm 2.

\noindent {\bf Case 2a: Arm~2 is suboptimal.} In this case, we
use $$\sqrt{a_1} \left ( \hat{\mu}_1^1(1) - \hat{\mu}_1^1(2) \right )
\geq \hat\Delta(2,1).$$ Combining this bound with~\eqref{eq:lemma4_1},
\begin{align*}
\prob{\G_2} &\leq \prob{\sqrt{a_1} \left ( \hat{\mu}_1^1(1) - \hat{\mu}_1^1(2) \right ) \geq \sqrt{a_2} \left ( \hat{\mu}_2^1(3) - \tau \right )} \\
&= \mathbb{P} \Bigl ( \{  \sqrt{a_1} \left ( \hat{\mu}_1^1(1) - \mu_1(1) - \hat{\mu}_1^1(2) + \mu_1(2) \right ) \\ &\qquad\qquad \geq   \sqrt{a_2} \left ( \hat{\mu}_2^1(3) -\mu_2(3) \right ) + \delta(1,3) + \delta(1,2) \} \Bigr ) \\
&\leq 6\exp \left (- \frac{n_1}{9} (\delta(1,3) + \delta(1,2))^2 \right ) \\
&\leq 6 \exp \left (- \frac{n_1}{9} \delta(1,3)^2 \right ).
\end{align*}

\noindent {\bf Case 2b: Arm~2 is a deceiver.} When arm 2 is a deceiver
arm, we use $$\sqrt{a_2} ( \tau - \hat\mu_2^1(2) ) \geq
\hat\Delta(2,1).$$ Combining this bound with~\eqref{eq:lemma4_1},
\begin{align*}
\prob{\G_2} &\leq \prob{\sqrt{a_2} \left( \tau - \hat\mu_2^1(2) \right) \geq \sqrt{a_2} \left ( \hat{\mu}_2^1(3) - \tau \right )}\\
&= \mathbb{P} \Bigl  (\{ ( \mu_2(3) - \hat{\mu}_2^1(3) ) - (\hat\mu_2^1(2) -\mu_2(2) )  \\ & \qquad\qquad \geq  ( \mu_2(2) - \tau )+ (\mu_2(3)- \tau) \} \Bigr ) \\
&\leq 4 \exp \left (- \frac{n_1}{4} a_2 (( \mu_2(2) - \tau )+ (\mu_2(3)- \tau))^2 \right ) \\
&\leq 4 \exp \left (- \frac{n_1}{4} \delta(1,3)^2 \right ) .
\end{align*}

Thus, combining the results of the two cases with~\eqref{eqn: 3 arm ub
  main} and~\eqref{eq:lemma4_2} gives us the following bound on the
probability of arm~1 being rejected at the end of the first round when
arm~2 is empirically optimal:
\begin{equation*}
  \prob{ \{\hat{J}(A_k) = 2\} \cap \mathcal{A}_1} \leq 10 \exp \left (- \frac{n_1}{9}  \Delta(1,3)^2 \right ).
  \label{eqn: 3 arm feasible ub deceiver}
\end{equation*}

\noindent {\bf Case 3: Arm 3 is an infeasible suboptimal arm.}

This case follows from Case~1 if $\delta(1,3)$ is dictated by the the
suboptimality gap of arm~3, and from Case~2 if $\delta(1,3)$ is
dictated by the the infeasibility gap of arm~3.

\ignore{
Next, we consider the case where arm 3 is an infeasible suboptimal
arm. Note that both \eqref{eqn: 3 arm ub feasible subopt main} and
\eqref{eqn: 3 arm feasible ub deceiver} are valid bounds for this
case. In the former, $\delta(1,3)=\sqrt{a_1}(\mu_1(3)-\mu_1(1)$, while
in the latter, $\delta(1,3)=\sqrt{a_2}(\mu_2(3)-\tau$. Thus, we have
that
\begin{align*}
&\prob{ \{\hat{J}(A_k) = 2\} \cap \mathcal{A}_1} \\&\leq 10 \exp \left (- \frac{n_1}{4}  \max \{ a_1(\mu_1(3)-\mu_1(1)^2), a_2(\mu_2(3)-\tau\}  )^2\right  ) \\
&=10 \exp \left (- \frac{n_1}{4}  \delta(1,3)^2 \right ). 
\end{align*} 
Thus, combining all three cases, we get that, for a feasible instance,
when arm 2 is empirically optimal, the probability of arm 1 being
rejected at the end of round 1 can be bounded in the following way:
\begin{align*}
\prob{ \{\hat{J}(A_k) = 2\} \cap \mathcal{A}_1} &\leq 10 \exp \left (- \frac{n_1}{9}  \delta(1,3)^2 \right ) \\
& \leq  10 \exp \left (- \frac{n_1}{9}  \Delta(1,3)^2 \right ) . \numberthis \label{eqn: 3 arm ub feasible final}
\end{align*}
}%end ignore
\end{proof}

%%%%%%%%%%%%%%%%%%%%%%%%%%%%%%%%%%%%%%%%
%%%%%%%%%%%%%%%%%%%%%%%%%%%%%%%%%%%%%%%%
%%%%%%%%%%%%%%%%%%%%%%%%%%%%%%%%%%%%%%%%

\subsection{Underlying instance is infeasible}

\begin{proof}[Proof of Lemma \ref{lemma:2arm-upper_bound}]
The infeasible instance consists of two arms, each of which has been
drawn $n_1$ times. Note that arm~1 is rejected either if arm~1 appears
empirically feasible, or if arm~2 appears favourable relative to arm~1
on the constraint attribute. Thus, $$\prob{\mathcal{A}_1} \leq
\prob{\hat\mu_2^1(1) \leq \tau} + \prob{\hat\mu_2^1(2) \leq
  \hat\mu_2^1(1)}.$$ Each of the two terms above can be bounded from
above using~\eqref{eqn: conc_inequality}, as demonstrated before, to
yield the statement of the lemma.

\end{proof}

%%%%%%%%%%%%%%%%%%%%%%%%%%%%%%%%%%%%%%%%
%%%%%%%%%%%%%%%%%%%%%%%%%%%%%%%%%%%%%%%%
%%%%%%%%%%%%%%%%%%%%%%%%%%%%%%%%%%%%%%%%

\begin{proof}[Proof of Lemma~\ref{lemma:3arm-upper_bound}]

The infeasible instance consists of three arms, each of which has been
drawn $n_1$ times. Let $\mathcal{B}_k$ denote the event that arm~1 is
empirically feasible at the end of round $k$. Proceeding similarly as
in \eqref{eqn: 2 arm ub main},
\begin{align*}
&\prob{ \{\hat{J}(A_k) = 2\} \cap \mathcal{A}_1} \\
&\leq \prob{   \mathcal{A}_1 \cap \mathcal{B}_1^{\mathsf{c}} \cap \{\hat{J}(A_k) = 2\} \cap \{\hat\mu_2(2), \hat\mu_2(3) > \tau \} } \\
&\quad +  6 \exp \left ( - n_{1} \Delta(1,3)^2 \right ), \numberthis \label{eqn: 3 arm ub main infeasible}
\end{align*}
the only difference being that here, we bound the probability of arm
1, 2 or 3 being empirically feasible using \eqref{eqn:
  conc_inequality}. We also use the fact that $\mu_2(1) < \mu_2(2)
\leq \mu_2(3)$. The first term in \eqref{eqn: 3 arm ub main
  infeasible} is the probability of arm 1 being rejected at the end of
round 1 when all three arms are empirically infeasible and arm 2 is
empirically optimal. It thus implies $\hat\mu_2^1(1) >\hat\mu_2^1(3),$
yielding
\begin{align*}
&\prob{   \mathcal{A}_1 \cap \mathcal{B}_1^{\mathsf{c}} \cap \{\hat{J}(A_k) = 2\} \cap \{\hat\mu_2(2), \hat\mu_2(3) > \tau \} } \\
&\quad \leq \prob{\hat\mu_2^1(1) >\hat\mu_2^1(3)} \\
&\quad \leq 4 \exp \left (- \frac{n_1}{4} (\mu_2(3) - \mu_2(1))^2 \right ) \\
&\quad = 4 \exp \left (- \frac{n_1}{4} \delta(1,3)^2 \right ).
\end{align*}
Combining the above result and \eqref{eqn: 3 arm ub main infeasible},
we get that
\begin{equation*}
\prob{ \{\hat{J}(A_k) = 2\} \cap \mathcal{A}_1} \leq 10 \exp \left ( - \frac{n_1}{4} \Delta(1,3)^2 \right ). \label{eqn: 3 arm ub infeasible final}
\end{equation*}
\ignore{
Thus, combining \eqref{eqn: 3 arm ub feasible final} and \eqref{eqn: 3
  arm ub infeasible final}, we get that the probability of rejecting
arm 1 in any three armed instance at the end of round 1 when arm 2 is
empirically optimal can be bounded in the following way:
\begin{equation*}
\prob{ \{\hat{J}(A_k) = 2\} \cap \mathcal{A}_1} \leq 10 \exp \left ( - \frac{n_1}{9} \delta(1,3)^2 \right ). 
\end{equation*}
} %end ignore
\end{proof}

%% file: sections/appendix_D.tex
\section{The Infeasible First algorithm} \label{app: IF}
Informally, the algorithm removes the most (empirically) infeasible arm that has survived so far. If there are no infeasible arms, it removes the most (empirically) suboptimal arm. The formal version of the algorithm is given in Algorithm \ref{algo: if}. 
  \begin{algorithm}[t]
  \caption{Infeasible First algorithm}
   \label{algo: if}
  \begin{algorithmic}[1]
    \Procedure{IF}{$T,K,\tau$}
    \State Let $A_1=\{1,\ldots,K\}$
    \State $\overline{\log}(K) := \frac{1}{2} + \sum_{i=2}^{K}\frac{1}{i}$ 
    \State $n_0 = 0,$ $n_k = \lceil \frac{1}{\overline{\log}(K) } \frac{T-K}{K+1-k} \rceil$ for $1 \leq k \leq K-1$
    
      \For{$k=1,\ldots,K-1$}
      \State For each $i \in A_k$, pull arm $i$ $(n_k-n_{k-1})$ times
      
      \State Compute $\hat{\mathcal{K}}(A_k)=\{i \in A_k : \hat{\mu}_{2}^{k}(i) > \tau \}$
       \State Compute $\hat{\mathcal{K}}^{\mathsf{c}}(A_k)=A_k \setminus \hat{\mathcal{K}}(A_k)$
       
             \If{ $\hat{\mathcal{K}}^{\mathsf{c}}(A_k) \neq \emptyset\}$ }
      	\State $A_{k+1}=A_k \setminus \{ \underset{i \in \hat{\mathcal{K}}^{\mathsf{c}}(A_k) }{\argmax} \hat{\mu}_2^{k}(i) \}$   (if there is a tie, choose randomly)
      \Else
      	\State $A_{k+1}=A_k \setminus \{ \underset{i \in \hat{\mathcal{K}}(A_k) }{\argmax} \hat{\mu}_1^{k}(i) \}$   (if there is a tie, choose randomly) 
	\EndIf
	
      \EndFor
      \State Let $\hat{J}_{T}$ be the unique element of $A_K$
      \If{$\hat{\mu}^{K-1}_{2}(\hat{J}_{T})> \tau$}
      	\State $\hat{O}([K])=0$
     \Else
     	\State $\hat{O}([K])=\hat{J}_{T}$
      \EndIf
      \State \textbf{return} $\hat{O}([K])$
    \EndProcedure
  \end{algorithmic}
  
\end{algorithm}

%% file: main.bbl
\begin{thebibliography}{27}
\providecommand{\natexlab}[1]{#1}

\bibitem[{Amani, Alizadeh, and Thrampoulidis(2019)}]{amani2019}
Amani, S.; Alizadeh, M.; and Thrampoulidis, C. 2019.
\newblock Linear stochastic bandits under safety constraints.
\newblock \emph{Advances in Neural Information Processing Systems}, 32.

\bibitem[{Audibert and Bubeck(2010)}]{audibert-bubeck}
Audibert, J.-Y.; and Bubeck, S. 2010.
\newblock {Best Arm Identification in Multi-Armed Bandits}.
\newblock In \emph{{COLT - 23th Conference on Learning Theory}}.

\bibitem[{Bhat and Prashanth(2019)}]{bhat2019}
Bhat, S.~P.; and Prashanth, L. 2019.
\newblock Concentration of risk measures: A Wasserstein distance approach.
\newblock In \emph{Advances in Neural Information Processing Systems},
  11739--11748.

\bibitem[{Bubeck and Cesa-Bianchi(2012)}]{bubeck2012}
Bubeck, S.; and Cesa-Bianchi, N. 2012.
\newblock Regret analysis of stochastic and nonstochastic multi-armed bandit
  problems.
\newblock \emph{Foundations and Trends{\textregistered} in Machine Learning},
  5(1): 1--122.

\bibitem[{Carpentier and Locatelli(2016)}]{carpentier2016tight}
Carpentier, A.; and Locatelli, A. 2016.
\newblock Tight (lower) bounds for the fixed budget best arm identification
  bandit problem.
\newblock In \emph{Conference on Learning Theory}.

\bibitem[{Cassel, Mannor, and Zeevi(2018)}]{cassel}
Cassel, A.; Mannor, S.; and Zeevi, A. 2018.
\newblock A General Approach to Multi-Armed Bandits Under Risk Criteria.
\newblock In Bubeck, S.; Perchet, V.; and Rigollet, P., eds., \emph{Proceedings
  of the 31st Conference On Learning Theory}.

\bibitem[{Chang(2020)}]{chang2020}
Chang, H.~S. 2020.
\newblock An asymptotically optimal strategy for constrained multi-armed bandit
  problems.
\newblock \emph{Mathematical Methods of Operations Research}, 1--13.

\bibitem[{Chang, Zhu, and Tan(2020)}]{chang2020risk}
Chang, J.~Q.; Zhu, Q.; and Tan, V.~Y. 2020.
\newblock Risk-constrained thompson sampling for cvar bandits.
\newblock \emph{arXiv preprint arXiv:2011.08046}.

\bibitem[{David et~al.(2018)David, Sz{\"o}r{\'e}nyi, Ghavamzadeh, Mannor, and
  Shimkin}]{david2018}
David, Y.; Sz{\"o}r{\'e}nyi, B.; Ghavamzadeh, M.; Mannor, S.; and Shimkin, N.
  2018.
\newblock {PAC} Bandits with Risk Constraints.
\newblock In \emph{ISAIM}.

\bibitem[{Drugan and Nowe(2013)}]{drugan2013}
Drugan, M.~M.; and Nowe, A. 2013.
\newblock Designing multi-objective multi-armed bandits algorithms: A study.
\newblock In \emph{The 2013 International Joint Conference on Neural Networks
  (IJCNN)}.

\bibitem[{Ehrgott(2005)}]{ehrgott2005}
Ehrgott, M. 2005.
\newblock \emph{Multicriteria optimization}, volume 491.
\newblock Springer Science \& Business Media.

\bibitem[{Hou, Tan, and Zhong(2022)}]{hou2022almost}
Hou, Y.; Tan, V.~Y.; and Zhong, Z. 2022.
\newblock Almost Optimal Variance-Constrained Best Arm Identification.
\newblock \emph{arXiv preprint arXiv:2201.10142}.

\bibitem[{Kagrecha, Nair, and Jagannathan(2020)}]{kagrecha2020constrained}
Kagrecha, A.; Nair, J.; and Jagannathan, K. 2020.
\newblock Constrained regret minimization for multi-criterion multi-armed
  bandits.
\newblock \emph{arXiv preprint arXiv:2006.09649}.

\bibitem[{Kagrecha, Nair, and Jagannathan(2019)}]{kagrecha2019}
Kagrecha, A.; Nair, J.; and Jagannathan, K.~P. 2019.
\newblock Distribution oblivious, risk-aware algorithms for multi-armed bandits
  with unbounded rewards.
\newblock In \emph{Advances in Neural Information Processing Systems},
  11269--11278.

\bibitem[{Karnin, Koren, and Somekh(2013)}]{Karnin2013}
Karnin, Z.; Koren, T.; and Somekh, O. 2013.
\newblock Almost optimal exploration in multi-armed bandits.
\newblock In \emph{International Conference on Machine Learning}.

\bibitem[{Kaufmann, Capp\'{e}, and Garivier(2016)}]{kaufmann16}
Kaufmann, E.; Capp\'{e}, O.; and Garivier, A. 2016.
\newblock On the Complexity of Best-Arm Identification in Multi-Armed Bandit
  Models.
\newblock \emph{J. Mach. Learn. Res.}, 17(1): 1–42.

\bibitem[{Kolla et~al.(2019)Kolla, Prashanth, Bhat, and
  Jagannathan}]{kolla2019}
Kolla, R.~K.; Prashanth, L.; Bhat, S.~P.; and Jagannathan, K. 2019.
\newblock Concentration bounds for empirical conditional value-at-risk: The
  unbounded case.
\newblock \emph{Operations Research Letters}, 47(1): 16--20.

\bibitem[{Lattimore and Szepesvári(2020)}]{lattimore}
Lattimore, T.; and Szepesvári, C. 2020.
\newblock \emph{Bandit Algorithms}.
\newblock Cambridge University Press.

\bibitem[{Locatelli, Gutzeit, and Carpentier(2016)}]{tbp-locatelli16}
Locatelli, A.; Gutzeit, M.; and Carpentier, A. 2016.
\newblock An optimal algorithm for the Thresholding Bandit Problem.
\newblock In Balcan, M.~F.; and Weinberger, K.~Q., eds., \emph{Proceedings of
  The 33rd International Conference on Machine Learning}, volume~48 of
  \emph{Proceedings of Machine Learning Research}, 1690--1698. New York, New
  York, USA: PMLR.

\bibitem[{Moradipari et~al.(2019)Moradipari, Amani, Alizadeh, and
  Thrampoulidis}]{moradipari2019}
Moradipari, A.; Amani, S.; Alizadeh, M.; and Thrampoulidis, C. 2019.
\newblock Safe linear {Thompson} sampling.
\newblock \emph{arXiv preprint arXiv:1911.02156}.

\bibitem[{Pacchiano et~al.(2021)Pacchiano, Ghavamzadeh, Bartlett, and
  Jiang}]{pacchiano2021}
Pacchiano, A.; Ghavamzadeh, M.; Bartlett, P.; and Jiang, H. 2021.
\newblock Stochastic bandits with linear constraints.
\newblock In \emph{International Conference on Artificial Intelligence and
  Statistics}, 2827--2835. PMLR.

\bibitem[{Sani, Lazaric, and Munos(2012)}]{sani2012}
Sani, A.; Lazaric, A.; and Munos, R. 2012.
\newblock Risk-aversion in multi-armed bandits.
\newblock In \emph{Advances in Neural Information Processing Systems},
  3275--3283.

\bibitem[{Tekin(2019)}]{tekin2019}
Tekin, C. 2019.
\newblock The biobjective multiarmed bandit: learning approximate lexicographic
  optimal allocations.
\newblock \emph{Turkish Journal of Electrical Engineering \& Computer
  Sciences}, 27(2): 1065--1080.

\bibitem[{Tekin and Tur{\u{g}}ay(2018)}]{tekin2018}
Tekin, C.; and Tur{\u{g}}ay, E. 2018.
\newblock Multi-objective contextual multi-armed bandit with a dominant
  objective.
\newblock \emph{IEEE Transactions on Signal Processing}, 66(14): 3799--3813.

\bibitem[{Vakili and Zhao(2016)}]{vakili2016}
Vakili, S.; and Zhao, Q. 2016.
\newblock Risk-averse multi-armed bandit problems under mean-variance measure.
\newblock \emph{IEEE Journal of Selected Topics in Signal Processing}, 10(6):
  1093--1111.

\bibitem[{Wang and Gao(2010)}]{wang2010}
Wang, Y.; and Gao, F. 2010.
\newblock Deviation inequalities for an estimator of the conditional
  value-at-risk.
\newblock \emph{Operations Research Letters}, 38(3): 236--239.

\bibitem[{Yahyaa and Manderick(2015)}]{yahyaa2015}
Yahyaa, S.; and Manderick, B. 2015.
\newblock Thompson sampling for multi-objective multi-armed bandits problem.
\newblock In \emph{Proceedings}, 47. Presses universitaires de Louvain.

\end{thebibliography}
